\documentclass{article}

\usepackage{microtype}
\usepackage{graphicx}
\usepackage{subfigure}
\usepackage{booktabs} 

\usepackage{hyperref}

\usepackage[accepted]{icml2023}

\usepackage{amsmath}
\usepackage{amssymb}
\usepackage{mathtools}
\usepackage{amsthm}

\usepackage[capitalize,noabbrev]{cleveref}

\theoremstyle{plain}
\newtheorem{theorem}{Theorem}[section]

\theoremstyle{definition}

\theoremstyle{remark}

\usepackage{xspace}
\usepackage{paralist}
\usepackage{colortbl}  
\newcommand{\vsa}{\vspace*{-0.28cm}}
\newcommand{\vsb}{\vspace*{-0.19cm}}
\newcommand{\vsc}{\vspace*{-0.16cm}}
\newcommand{\cbit}{\begin{compactitem}}
	\newcommand{\ceit}{\end{compactitem}}
\newcommand{\cben}{\begin{compactenum}}
	\newcommand{\ceen}{\end{compactenum}}
\newcommand{\mourmeth}{\text{$D^2$Match}}
\newcommand{\ourmeth}{$\mourmeth$\xspace}
\newcommand{\eat}[1]{}
\newcommand{\diag}{\mbox{diag}}
\newcommand{\aggregate}{\mbox{AGG}}
\newcommand{\concat}{\mbox{concat}}
\usepackage{color}

\usepackage{xcolor}
\definecolor{Gray}{gray}{0.95}
\definecolor{Cyan}{rgb}{0.88,1,1}
\usepackage[textsize=tiny]{todonotes}

\usepackage{amsmath,amsfonts,bm}

\def\eqref#1{equation~\ref{#1}}

\def\1{\bm{1}}

\def\ervq{{\textnormal{q}}}

\def\ervt{{\textnormal{t}}}

\DeclareMathAlphabet{\mathsfit}{\encodingdefault}{\sfdefault}{m}{sl}
\SetMathAlphabet{\mathsfit}{bold}{\encodingdefault}{\sfdefault}{bx}{n}

\def\gC{{\mathcal{C}}}

\def\gQ{{\mathcal{Q}}}

\def\gS{{\mathcal{S}}}
\def\gT{{\mathcal{T}}}

\usepackage[textsize=tiny]{todonotes}

\icmltitlerunning{D2Match: Leveraging Deep Learning and Degeneracy for  Subgraph Matching}

\begin{document}

\twocolumn[
\icmltitle{D2Match: Leveraging Deep Learning and Degeneracy for Subgraph Matching}



\icmlsetsymbol{intern}{$\dagger$}
\icmlsetsymbol{ca}{*}

\begin{icmlauthorlist}
\icmlauthor{Xuanzhou Liu}{thusz,idea,intern}
\icmlauthor{Lin Zhang}{idea}
\icmlauthor{Jiaqi Sun}{thusz}
\icmlauthor{Yujiu Yang}{thusz,ca}
\icmlauthor{Haiqin Yang}{idea,ca}

\end{icmlauthorlist}

\icmlaffiliation{thusz}{Shenzhen International Graduate School, Tsinghua University, Shenzhen, China}
\icmlaffiliation{idea}{International Digital Economy Academy (IDEA). $\dagger$Work done when Xuanzhou was interned at IDEA}

\icmlcorrespondingauthor{Yujiu Yang}{yang.yujiu@sz.tsinghua.edu.cn}
\icmlcorrespondingauthor{Haiqin Yang}{hqyang@ieee.org}

\icmlkeywords{Machine Learning, Graph Neural Networks, Subgraph Matching, ICML}

\vskip 0.3in
]



\printAffiliationsAndNotice{}  

\begin{abstract}
Subgraph matching is a fundamental building block for graph-based applications and is challenging due to its high-order combinatorial nature.  Existing studies usually tackle it by combinatorial optimization or learning-based methods.  However, they suffer from exponential computational costs or searching the matching without theoretical guarantees. 
In this paper, we develop \ourmeth by leveraging the efficiency of Deep learning and Degeneracy for subgraph matching. 
More specifically, we first prove that subgraph matching can degenerate to subtree matching, and subsequently is equivalent to finding a perfect matching on a bipartite graph.  We can then yield an implementation of linear time complexity by the built-in tree-structured aggregation mechanism on graph neural networks.  Moreover, circle structures and node attributes can be easily incorporated in \ourmeth to boost the matching performance. 
Finally, we conduct extensive experiments to show the superior performance of our \ourmeth and confirm that our \ourmeth indeed exploits the subtrees and differs from existing GNNs-based subgraph matching methods that depend on memorizing the data distribution divergence.
\end{abstract}

\section{Introduction}

Subgraph isomorphism, or subgraph matching at the node level~\citep{McCreesh2018SubIso}, is a critical yet particularly challenging graph-related task.  It aims to determine whether a query graph is isomorphic to a subgraph of a large target graph.  It is an essential building block for various graph-based scenarios, e.g., alignment of cross-domain data~\citep{Chen2020GOT}, temporal alignment of time series~\citep{zhou2009canonical}, and motif matching~\citep{Milo2002Motif,Peng2020MotifMatchingBS}, etc. 

Existing studies for subgraph matching can be divided into two main streams: combinatorial optimization (CO)-based and learning-based methods~\citep{Vesselinova2020COLearningGraph}.
Early algorithms often formulate subgraph matching as a CO problem that aims to find all exact matches in a target graph. Unfortunately, this leads to an NP-complete issue~\citep{Ullmann1976,Cordella2004VF2}, which suffers from exponential time costs.
To alleviate the computational cost, researchers have employed approximate techniques to seek inexact solutions~\citep{Mongiov2010SigmaAS,Yan2005Graphfil,Shang2008QuickSI}, which yield suboptimal matchings. 
An alternative solution is to frame subgraph matching as a supervised learning  problem~\citep{Bai2019SimGNN,Rex2020NeuroMatch,Bai2020GraphSim}, which utilizes the \eat{intrinsic properties of} Graph Neural Networks (GNNs).  However, the learning-based or GNN-based methods mainly aim to optimize the representations in a black-box way.  The lack of theoretical guarantees makes them inexplicable and often cannot seek the exact match subgraphs.   

In order to tackle the above challenges, we propose a white-box GNN-based solution, \ourmeth, to leverage the efficiency of Deep GNNs and Degeneracy for subgraph matching.  With rigorous theoretical proofs, we provide explainable results at each learning step.  We first prove that finding the matched nodes between the query graph and the target one can degenerate to check whether the corresponding subtrees rooted at these two nodes are subtree isomorphic.  This  degeneration allows us to check whether a perfect matching exists on a bipartite graph, which results in a polynomial time complexity  solution. The above two steps convert subgraph matching into computing an indicator matrix whose elements represent the subtree isomorphic relationship between nodes in the query graph and the target one. Hence, the matching matrix required by the CO-based methods for subgraph matching degenerates to seeking an indicator matrix, which is computed by GNNs via its intrinsic tree-structured aggregation mechanism.

Note that this implementation battles the matching mechanism directly by GNNs rather than optimizing the representations of GNNs as in the existing work.
Our implementation allows us to reduce the time cost of perfect matching from polynomial time to linear time.  Moreover, we can easily incorporate other information, including circle structures (abstracted as {\em supernodes}) and node attributes,  into our \ourmeth to boost the matching performance.

Our contribution is four-fold: (1) We propose the first white-box GNN-based model, called \ourmeth, to leverage deep learning and degeneracy for subgraph matching.  We provide rigorous theoretical proofs to guarantee that subgraph matching can degenerate to subtree matching, and finally perfect matching on a bipartite graph.
{(2) To the best of our knowledge, this is the first GNN-based model to tackle subgraph matching directly, which degenerates the matching matrix required by the CO-based methods to an indicator matrix computed by GNNs via the intrinsic tree-structured aggregation mechanism.  This allows us to compute the indicator matrix in linear time. }  (3) Our \ourmeth can easily include other information, including circle structures and node attributes to boost the model performance.  (4) Extensive empirical evaluations show that \ourmeth outperforms state-of-the-art subgraph matching methods by a substantial margin, and uncovered that learning-based methods tend to capture the divergence of the data distribution rather than exploiting graph structures.

\begin{figure*}[tp]
\begin{center}
\centerline{\includegraphics[trim={0cm 7.8cm 0cm 0cm},clip, width=1.0\textwidth]{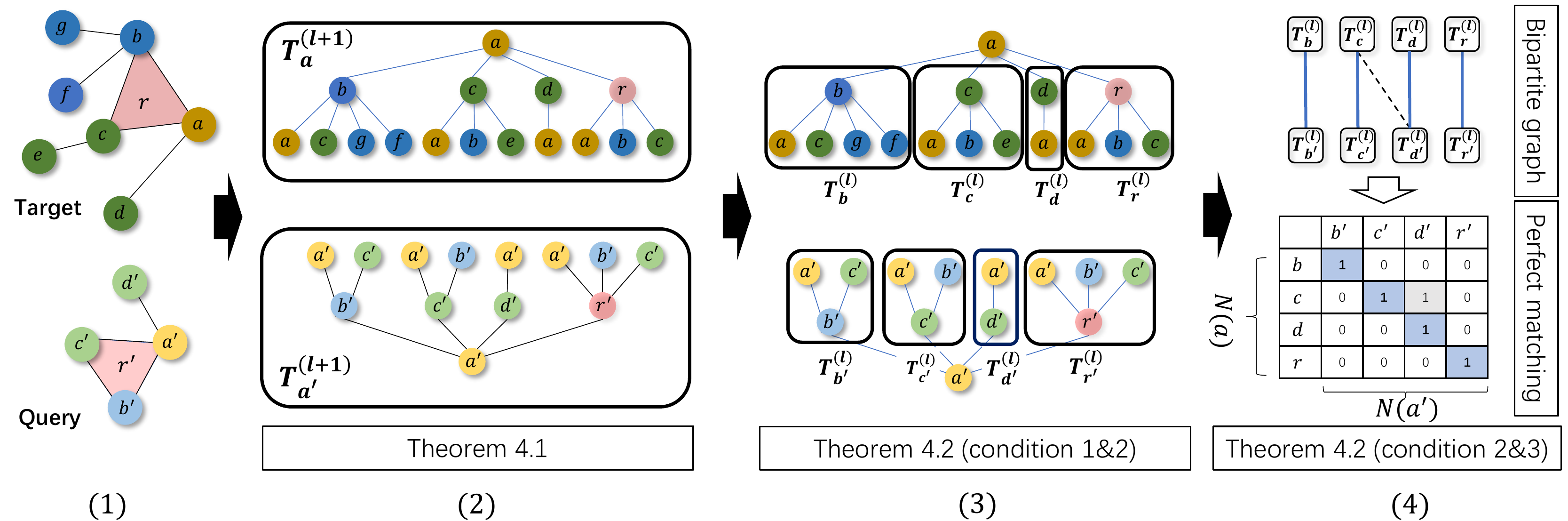}}
\caption{ \footnotesize 
Illustration of the proposed degeneracy procedure for subgraph matching: the problem of whether node $a'$ in the query graph matches node $a$ in the target one {\em degenerate} to check whether the constructed $(l+1)$-depth subtrees rooted at node $a$ and $a'$ are subtree isomorphic.  This corresponds to Step (2) and is guaranteed by Theorem~\ref{thm:subtree_isomorphim}.  The procedure is to check whether $T_{a}^{(l+1)} $ and $T_{a'}^{(l+1)}$ are subtree isomorphic.  As shown in Step (3) and guaranteed by conditions 1\&2 in Theorem~\ref{thm:equivalent_properties}, this is equivalent to checking whether subtree isomorphism holds for every $l$-depth subtrees rooted in $N(a')$ to a unique $l$-depth subtree in $N(a)$, where $N(\cdot)$ is the neighbor set of a given node.  After that, based on conditions 2\&3 in Theorem~\ref{thm:equivalent_properties}, the problem of subtree isomorphic is equivalent to checking whether there is perfect matching on the bipartite graph from every $l$-depth subtrees rooted in $N(a')$ and $N(a)$, respectively.  This corresponds to Step (4), where the upper part represents the constructed bipartite graph, and the lower part is its adjacency matrix.  By running the Hall's marriage algorithm, we can determine the perfect matching, where the selected edges are highlighted by the blue areas.  The computation procedure for this indication matrix is implemented via the intrinsic tree-structured aggregation mechanism on GNNs, which is proved in  Theorem~\ref{theo:znw} and guarantees the linear time cost as analyzed in Sec.~\ref{sec:complexity}.
}
\label{fig:introduction}
\end{center}
\vskip -0.2in
\end{figure*}

\vsa 
\section{Related work}
Subgraph matching is to check whether a query graph is subgraph isomorphic to the target one~\citep{McCreesh2018SubIso}.  Here, we highlight three main lines of related work:  

\textbf{Combinatorial optimization (CO)-based methods} first tackle subgraph matching by only modeling graph structure~\citep{Ullmann1976}. 
Some later work starts to facilitate both graph structure and node attributes~\citep{he2008GraphQL,Shang2008QuickSI,Han2013TurboisoTU, Bhattarai2019CECI}.  
These combinatorial optimization methods often rely on backtracking~\citep{PriestleyW94}, i.e., heuristically performing matching on each pair of nodes from the query and the target graphs.  Such methods suffer from exponential computing costs.  A mitigated solution is to employ an inexact matching strategy.  
Early methods first define metrics to measure the similarity between the query graph and the target graph.  
Successive algorithms follow this strategy and propose more complex metrics.  For example,~\citet{Mongiov2010SigmaAS} convert the graph matching problem into a set-cover problem to attain a polynomial complexity solution.
~\citet{Yan2005Graphfil} introduce a thresholding method to filter out mismatched graphs.
~\citet{Khan2011Neighbor} define a metric based on neighborhood similarity and employ an information propagation model to find similar graphs. \citet{Kosinov2002Eigen} and \citet{Caelli2004Eigen} align the nodes' eigenspace and project them to the eigenspace via clustering for matching.
However, most of these algorithms cannot scale to large graphs due to the high computational cost, and their hand-crafted features make them hard to generalize to complex tasks.

\textbf{Learning-based methods} typically compute the similarity between the query and target graphs, e.g., comparing their embedding vectors.
~\citet{Bai2019SimGNN} adopt GNNs to learn representations of the graphs and employs a neural tensor network to match the representation of graph pairs.
One immediate challenge is that a single graph embedding vector cannot capture the partial order of subgraph isomorphism. 
Thus,~\citet{Rex2020NeuroMatch} train a GNN model to represent graphs while incorporating order embeddings to learn the partial order.  These methods can compute graph-level representations, achieving high computational efficiency.  However, they miss the node-level information, which may lose critical details in subgraph matching. 
To perform node-level matching, several methods~\citep{Bai2020GraphSim,Li2019GMN} introduce the node-level representation into the \eat{node-level matching} problem.
These methods often adopt different attention mechanisms to generate pairwise relations.
However, abusing the attention mechanism makes the model lack interpretability and theoretical guarantee.
Others transform the subgraph matching problem into an edge matching problem and generate prediction results through the matching matrix obtained by Sinkhorn's algorithm~\citep{Roy2022IsoNet}, thereby providing interpretability for the model.
The process of turning node matching into edge matching, however, loses necessary information about edges' relation, such as edges' common nodes, which hurts the expressibility of the model.

\textbf{Graph Neural Networks (GNNs)} are powerful techniques~\citep{xu2019powerful, kipf2017semisupervised} yielding breakthroughs in many key applications~\citep{hamilton2018inductive}. 
{Graph neural networks mostly iteratively aggregate information that can be expressed as follows,
\begin{equation}
\begin{aligned}
    &H^{(l+1)} = \aggregate_{f}(A,H^{(l)}) \\
    &H_{i:}^{(l+1)} = \phi\left(H_{i:}^{(l)}, f\left(\left\{H_{j:}^{(l)}: j \in \mathcal{N}(i)\right\}\right)\right)
\end{aligned}
\end{equation}
where $f(\cdot)$ is the aggregation function such as mean or max; $\phi$ is the update function. $H^{(l)}$ is the representation matrix and its $i$-th row, denoted by $H_{i:}$, is the representation of node $i$.
}
Over the last few years, there is considerable progress in proposing different ways of aggregating.  For example, 
GraphSAGE~\citep{hamilton2018inductive} aggregates node features with mean/max/LSTM pooled neighboring information. 
Graph Attention Network (GAT)~\citep{Petar2018graph} aggregates neighbor information using learnable attention weights.  A close work for subgraph matching is Graph Isomorphism Network (GIN)~\citep{xu2019powerful}, which converts aggregation as a learnable function based on the Weisfeiler-Lehman (WL) test to maximize the performance of GNNs. 
{However, the WL test~\citep{xu2019powerful} cannot address the subgraph matching problem because it hashes the tree structure and ignores the partial order information for subgraph matching.}
\vsa

\section{Preliminary}
We define some notations accordingly: Let $A_\gQ$ and $A_\gT$ be the adjacency matrix of the query graph $G_\gQ$ and the target graph $G_\gT$, respectively.  $N(\cdot)$ denotes the neighbor set of a given node or a given set.  $|\cdot|$ denotes the cardinality of a set.  {$T_{v}^{(l)}$ defines the subtree whose root is $v$ and expands up to $l$-hop neighbors of $v$ (or $l$-depth of the subtree).  In the paper, the concepts of the $l$-hop neighbors and the $l$-depth subtrees are interchangeable. }

{\bf Problem Definition:}
Suppose we are given a query graph, $G_\gQ(V_\gQ,E_\gQ)$, and a target graph, $G_\gT(V_\gT,E_\gT)$. Here, $(V_\gQ,E_\gQ)$ and $(V_\gT,E_\gT)$ are the pairs of vertices and edges related to the query graph and the target graph, respectively.  Besides, the node attributes of the query graph and the target graph are denoted as 
${X}_\gQ\in  \mathbb{R}^{ \left|V_\gQ\right| \times D}$ and  ${X}_\gT\in  \mathbb{R}^{ \left|{V}_\gT\right|  \times D}$, respectively, where $D$ is the dimension of the node attributes. The problem of subgraph matching is to identify whether the query graph $G_\gQ$ is subgraph isomorphic to the target graph $G_\gT$, i.e. if there exists an injective $\xi: V_\gQ \to V_\gT$ such that \eat{$\forall u \in V_\gQ, (X_\gQ)_{u} = (X_\gT)_{f(u)}$ and} $\forall u, v \in V_\gQ, (u,v)\in E_\gQ \Leftrightarrow (\xi(u),\xi(v))\in E_\gT$.

The injective can be represented by a matching matrix, $S\in \{0, 1\}^{|V_\gT| \times |V_\gQ|}$, where $S_{ij}=1$ if and only if node pair$(i,j)$ is matched, i.e., $\xi(j) = i$. Therefore, subgraph isomorphism is equivalent to checking the existence of such a matrix $S$, i.e., whether there exists assignment matrix $S$ such that:
\begin{equation}\label{eq:indicate_cond}
    SA_\gQ S^T = A_\gT, \quad \mbox{s.t.} 
\begin{cases}
\sum_{i}{S_{ij}} = 1, \forall j\\
\sum_{j}{S_{ij}} \le 1, \forall i\\
S_{ij} \in \{0, 1\}, \forall i,j.
\end{cases}
\end{equation}
We denote $G_\gQ \subset G_\gT$ if $G_\gQ$ is a subgraph of $G_\gT$.  Here, we first define several key concepts:

{\bf WL Subtree:}
The Weisfeiler-Lehman (WL) test is an approximate solution to the graph isomorphism problem with linear computational complexity~\citep{shervashidze2011weisfeiler}.
The WL test performs the aggregation on nodes' labels and their neighborhoods recursively, followed by hashing the aggregated results into unique new labels.
As a result, this test produces an unordered tree for each node, called the WL subtree, which is a balanced tree with the height of the number of iterations.
After repeating the algorithm $k$ times, the obtained WL subtree for a node includes the structural information of the $k$-hop subgraph from that node.
Research shows that the expressiveness of GNNs with the message-passing mechanism is upper-bounded by the WL test~\citep{xu2019powerful}.

{\bf Subtree Generation:}
For any node $v$ in a graph, one can obtain a subgraph $Sub^{(l)}_v$ by taking the  $l$-hop neighbor of $v$. 
Given any tree generation method, e.g., the WL subtree, we always obtain its corresponding subtree whose root is $v$: \vsa  
\begin{equation}
    \begin{aligned}
     T_v^{(l)}=\Psi(Sub_v^{(l)}),  
    \end{aligned}
    \label{eq:tree_generation}
\end{equation} 
where $\Psi$ is a subtree generation function. 
Unless stated otherwise, we employ the WL subtree to generate subtrees for a given node due to its uniqueness~\citep{Xu2018}.
Instead of explicitly constructing such trees, we can run GNNs in a graph, since building a $k$-order WL subtree is equivalent to aggregating $k$ times in GNNs~\citep{Xu2018}.
Notice that traditional methods such as Breadth-First-Search (BFS) and Depth-First-Search (DFS) are not applicable at this work because they do not satisfy the uniqueness property.  In particular, the tree generated for the same node by BFS or DFS will be different due to different search order.

{\bf  Perfect Matching in Bipartite Graphs:}
A perfect matching~\citep{gibbons1985algorithmic} is a chosen edge set of a graph wherein every node of the graph is incident to exactly one edge.
Hence, according to the edge set, each node in the graph corresponds to only one other node.
The existence of a perfect matching on a bipartite graph can be resolved via Hall's Marriage Theorem~\citep{Hall1935OnRO}.
\begin{theorem}
(Hall's marriage theorem)
  Given a bipartite graph, $B(X, Y, E)$ that has two partitions: $X$ and $Y$ and $|Y|\le |X|$, where  $E$ denotes the edges.
  The necessary and sufficient condition of the existence of the perfect matching in $B(X, Y, E)$ is : $\forall~W\subseteq Y,|W|\le |N(W)|$, where $N(W)$ is the neighborhood of $W$ defined by $N(W)=\{b_j\in X: \exists a_i\in W, (a_i,b_j)\in E\}$. 
\label{theo:Hall}
\end{theorem}

\section{The Proposed Method}
In the following, we present our proposed \ourmeth with the theoretical derivation and its extensions.

\subsection{Subgraph Matching Degeneracy}\eat{On the Degeneracy of the Subgraph Matching Problem}
\label{sec:graph-str}

We approach the subgraph matching problem from a degeneracy perspective, framing this problem as a subtree matching problem with linear complexity.
A fundamental question to the subgraph matching problem is on what conditions one subgraph is isomorphic to the other.
Since the subgraph matching problem is NP-complete, the exact answer to this question becomes impractical.
Instead, we can reduce the answer of finding both sufficient and necessary conditions to that of necessary only. 
What follows is to construct a criterion that any isomorphic pairs can meet.

Subtree isomorphism is a special task of subgraph isomorphism and can be attained in polynomial time.  Here, we aim to reduce the subgraph isomorphism problem to a subtree isomorphism problem. 
Inspired by~\citet{xu2019powerful}, we construct a criterion by checking the subtrees rooted at the nodes: if $G_\gQ$ is a subgraph of $G_\gT$, then the tree rooted at any node $q$ of $G_\gQ$ should be a subtree of the tree rooted at the matched node $t=\xi(q)$ of $G_\gT$.  As stated in the following theorem, some additional properties of these trees are needed to make this criterion hold:
\begin{theorem}\label{thm:subtree_isomorphim}
Given a target graph $G_\gT(V_\gT,E_\gT)$ and a query graph $G_\gQ(V_\gQ,E_\gQ)$, 
if $G_\gQ \subset G_\gT$, and the subtree generation function $\Psi$ as defined in Eq.~(\ref{eq:tree_generation}) meets the following  condition: 
\begin{equation}
\forall \ \mbox{graph pair} \ (G_\gS, G), \mbox{if}\  G_\gS\subset G, \mbox{then}\ \Psi(G_\gS)\subset \Psi(G),
\end{equation}
then the subgraph isomorphic mapping  $\xi\!:\!V_\gQ\to V_\gT$ ensures the $l$-depth subtrees of the subgraph are isomorphic to the subtrees of the corresponding subgraph: 
\begin{equation}
\forall l\ge1,\ervq\in V_\gQ,  \ervt = \xi(\ervq)\in V_\gT \Rightarrow  T_\ervq^{(l)}\subset T_{\ervt}^{(l)},  
\label{eq:necessary}
\end{equation}
\label{theo:necessary}
\end{theorem} \vsa 
Please find the proof in Appendix~\ref{theo:necessary_app}. 
This theorem provides a necessary condition for the potential isomorphic pairs.\eat{, i.e., those who pass the test. 
Given a query graph and a target graph, we can construct an indicator matrix $S\in R^{|V_\gT|\times |V_\gQ|}$ by setting $S_{\ervt \ervq}^{}$ to $1$ when $T_\ervq^{} \subset T_\ervt^{}$ and $0$ otherwise.  The isomorphic test becomes to check the validity of Eq.~(\ref{eq:indicate_cond}).}
{With this theorem,  given a query graph $G_\gQ$ and a target graph $G_\gT$, the node $\ervq$ from $G_\gQ$ is possible to match node $\ervt $ in $G_\gT$ only if $T_\ervq^{(l)} \subset T_\ervt^{(l)}$. 
It then becomes an isomorphic test, which checks whether there exists an assignment such that each node  $\ervq$ in $G_\gQ$ possibly corresponds to one unique node $\ervt$ in $G_\gT$.
Forming a boolean indicator matrix $S^{(l)}\in R^{|V_\gT|\times |V_\gQ|}$:
\begin{equation}
S^{(l)}_{\ervt \ervq} = 
\begin{cases}
1, \mbox{if} \ T_\ervq^{} \subset T_\ervt^{}\\
0, \mbox{otherwise}
\end{cases},
\end{equation}
we arrive at checking whether the indicator matrix $S^{(l)}$ contains a valid assignment matrix $S$ satisfies Eq.~(\ref{eq:indicate_cond}).
}

Due to the favorite property of uniqueness, we employ the WL subtree as the generation function. What follows is how to determine the subtree isomorphism relationship between $l$-order trees. To be specific, based on WL subtree generation, we solve the subtree matching problem iteratively, intending to find a perfect matching on a bipartite graph in each iteration. The following theorem guarantees the conversion:

\begin{theorem}\label{thm:equivalent_properties}
Given a node $\ervq$ in the query graph and a node $\ervt$ in the target graph, the following three conditions are equivalent:
\begin{compactenum}[1)]
\item $T_\ervq^{(l+1)}\subset T_\ervt^{(l+1)}$.
\item There exists an injective function on the neighborhood of these nodes as $f\!:\!N(\ervq)\to N(\ervt)$, s.t. $\forall \ervq_i\in N(\ervq), \ervt_i=f(\ervq_i), T_{\ervq_i}^{(l)}\subset T_{\ervt_i}^{(l)}$. 
\item  There exists a perfect matching on the bipartite graph $B^{(l)}(N(\ervt), N(\ervq), E)$, where $\forall \ervt_j\in N(\ervt), \ervq_i\in N(\ervq), (\ervt_j, \ervq_i)\in E$ if and only if $T_{\ervq_i}^{(l)}\subset T_{\ervt_j}^{(l)}$.
\end{compactenum}
\label{theo:recursive_equal_main}
\end{theorem}

The proof is provided in Appendix~\ref{theo:recursive_equal_app}. The equivalence of the first two conditions implies that matching subtrees of a pair of nodes is equivalent to matching all subtrees from their child nodes.  
As a result, the indicator matrix \eat{needs to} {can} be updated recursively.  That is, the indicator  matrix at the ($l+1$)-th layer, i.e., $S^{(l+1)}$, should rely on $S^{(l)}$.
Meanwhile, the equivalence of the last two conditions means that matching the subtrees from these child nodes is equivalent to solving the perfect matching on the corresponding bipartite graph whose nodes represent the subtrees of the child nodes. 
In summary, Theorem~\ref{theo:recursive_equal_main} tells us that subgraph matching is equivalent to delivering perfect matching on a bipartite graph.  A visualization of this procedure is shown in Fig.~\ref{fig:introduction}.

Motivated by Hall's marriage Theorem~\ref{theo:Hall}, we develop an efficient algorithm to address the perfect matching procedure.
A straightforward solution is to randomly select a subset {$W$} from the smaller partition of $B^{(l)}(N(\ervt), N(\ervq), E)$, i.e., $W \subseteq N(\ervq)$, and count whether $W$'s neighbors, i.e., $N(W) \subseteq N(\ervt)$ in the other partition, have more elements than $W$.
After repeating this process multiple times, we obtain a perfect matching when no instance violates the criterion. Note that we perform the procedure for all possible matched node pairs.

\textit{Is it possible to execute all pairs in parallel?}
Luckily, we can borrow GNNs to accomplish the perfect matching.
Specifically, when computing a perfect matching between node $\ervq \in G_\gQ$ and node $\ervt \in G_\gT$, one needs to find $W$ such that it satisfies $W\subseteq N(\ervq)$ according to Theorem~\ref{theo:Hall}.
In practice, we can obtain this by sampling the neighbors of node $\ervq$, equating to sampling the edges, or the Drop Edge operation~\citep{Hamilton2017Sage}.
In this way, we obtain a sampled graph $\tilde{G}_\gQ$ from the query graph $G_\gQ$, along with its adjacency matrix $\tilde{A}_\gQ$. 
Following Theorem~\ref{theo:Hall}, we conclude that $W=N'(\ervq)$ with $N'(\ervq)\subset N(\ervq)$ since node $\ervq$'s neighbors in $\tilde{G}_\gQ$ are a subset of the original graph.
At each iteration, we will perform the counting w.r.t. $W$ and its neighbor set $N(W)$ for each node pair $(\ervt, \ervq)$, and check whether $|N(W)|\ge |W|$ holds. 
To be efficient, we define a binary matrix,  $\Phi^{(l)}(\tilde{A}_\gQ, A_\gT)\in R^{|V_\gT|\times |V_\gQ|}$, where its element at $(\ervt, \ervq)$ corresponds to the result of the node pair $(\ervt, \ervq)$ between sampled query graph $\tilde{G}_\gQ$ and target graph $G_\gT$.

Based on Theorem~\ref{theo:necessary}, we need to update the indicator matrix $S^{(l)}$ recursively,
\eat{making the update of $\Phi$ executed in recursion accordingly.}
{where we compute $\Phi^{(l)}(\tilde{A}_\gQ, A_\gT)$ with multiple sampled $\tilde{G}_\gQ$.}
We next show that computing $\Phi$ is equivalent to performing the GNN-based aggregation on the related graphs for any given $S^{(l)}$.

\begin{theorem}
Given the sampled query graph and the target graph, we can construct their adjacency matrices , $\tilde{A}_\gQ$ and $A_\gT$, and the degree matrix of the sampled query graph $\tilde{D}_\gQ=\mbox{diag}(\sum_{s}((\tilde{A}_\gQ)_{:s}))$. Here, we denote the indicator matrix at the $l$-th hop as $S^{(l)}$.  To check the validity of $|N(W)|\ge |W|$ for each node pair, we can check whether each element of $\Phi^{(l+1)}(\tilde{A}_\gQ, A_\gT)$ is true or not, where $\Phi^{(l+1)}(\tilde{A}_\gQ, A_\gT) \coloneqq Z_{N(W)}\ge 1 $, $Z_{N(W)}= \aggregate_{\scriptsize\mbox{sum}}(A_\gT,Z_W^T)$ and $  Z_W=\aggregate_{\max}(\tilde{D}_\gQ^{-1}\cdot\tilde{A}_\gQ,(S^{(l)})^T).$
\label{theo:znw}
\end{theorem}
The proof is provided in Appendix~\ref{theo:znw_app}.
Recalling Theorem~\ref{theo:Hall}, we need to check $|N(W)|\ge |W|$ for each node pair $(\ervt, \ervq)$, i.e., to check whether each element of $\Phi$ is true for each sampled $\tilde{A}_\gQ^{(k)}$. 
The condition is valid only when $\Phi$ is true for each {sample}, i.e., $\tilde{G}_\gQ^{(k)}$.  Hence we can check the criterion by the following element-wise product: \vsa 
\begin{equation}
    \begin{aligned}
    \scriptsize
        S^{(l+1)}_{subtree}=\bigodot_{k=0}^{K}{\Phi^{(l+1)}(\tilde{A}_\gQ^{(k)},A_\gT)},
    \end{aligned}
\end{equation}
where $\bigodot$ denotes the element-wise multiplication between matrices.
In practice, $\Phi^{(l+1)}(\tilde{A}_\gQ^{(k)},A_\gT)$ considers three cases:
\begin{equation}\label{eq:Phi}
    \footnotesize
\begin{cases}
       \aggregate_{\scriptsize\mbox{sum}}(A_\gT,  \aggregate_{\max}({D}_\gQ^{-1}{A}_\gQ,(S^{(l)}))^T  )\ge 1 & \! k=0 \\
       \aggregate_{\scriptsize\mbox{sum}}(A_\gT,  \aggregate_{\max}(\tilde{D}_\gQ^{-1}\tilde{A}_\gQ^{(k)},(S^{(l)}))^T  )\ge 1 &\!  k\in[1,K)\\
       \aggregate_{\min}(A_\gQ,\aggregate_{\max}( A_\gT,(S^{(l)})^T)\ge 1 & \! k=K
\end{cases}
\end{equation}
The above three cases allow us to balance the computation cost and accuracy.
\eat{To be efficient, we take the downsampling for $ k\in[1,K-1]$.}
Initially, when $k=0$, we deliver a full-size sampling for all nodes to avoid induction bias. 
When $k=K$, we perform the single-node sampling such that no node is omitted. 
The cases of $k\in[1,K-1]$ are computed via downsampling.

We want to highlight the difference between ours and other learning-based methods regarding GNNs.
Here we employ a GNN model to accomplish the procedure of subtree matching, along with theoretical equivalence.
Unlike performing the matching in our model, prior learning-based models use GNNs to capture the variance of data distribution for similarity inference since deep learning models learn distributional information to distinguish samples from different classes.

\subsection{Boosting the Matching}
It is noted that our \ourmeth cannot guarantee that all isomorphism pairs are selected precisely.  In order to boost the model performance, in the following, we design the corresponding mechanism to include more information, i.e., circle structures and node attributes to attain more precise isomorphism pairs.

\subsubsection{Dealing with Circles  }
\label{sec:circles}

Prior methods often ignore circle structures though they are common in graphs and critical to be handled.  
In particular, learning-based methods rely on the expressibility of GNNs, which have difficulty identifying cyclic structures due to their WL-tree like aggregation mechanism~\citep{sato2020survey}.
The underlying idea of our \ourmeth is to select certain set of circles and construct the circle structures as {\em supernodes}.  This allows us to formulate the circle matching as a standard subtree matching problem.  Before detailing our strategy, we first present two desired properties of the selected set of circles in a graph.

{{\bf Atomic:} Let $c=(v_1,...,v_{l(c)},v_1)$ define a circle and $v(c)$ be the set of nodes of circle $c$, a circle is an {\em atomic} circle if it does not contain a smaller circle.  That is, there is no circle $c'$ such that $v(c')\subset v(c)$. A circle set $\gC$ is {\em atomic} if every circle in the set is {\em atomic}. }

{\bf Consistency:} Each query circle must correspond to one circle in the target graph if the query graph is a subgraph of the target graph , i.e., 
$\forall c_\gQ=(v_1,...,v_{l},v_1)\in \gC_\gQ, \exists c_\gT=(\xi(v_1),...,\xi(v_{l}),\xi(v_1)) \in \gC_\gT$ where $\xi$ is the subgraph isomorphic mapping.
 $\gC_\gQ$ and $\gC_\gT$ are the selected circle sets of the query graph and the target graph, respectively.

The atomic property aims to ensure the compactness of circles, and the consistency attempts to ensure that the relation between a query and a target set of circles is injective.
These two properties ensure a well-qualified set for matching.
In practice, we can take advantage of {\em chordless cycles}~\citep{west_introduction_2000}, \eat{such as carbon rings in chemistry}
 to serve our goal of matching circles.
We now state our theorem below to show that these cycles \eat{can be used for matching.}{satisfy the above consistency and atomic property.} 

\begin{theorem}
Every chordless cycle is atomic. Every chordless cycle $c_\gQ$ in an induced subgraph of the original query graph $G_\gQ$ must correspond to a chordless cycle $c_\gT$ in the origin graph $G_\gT$. 
\end{theorem} 
The proof is provided in Appendix,~\ref{theo:circles}.
This theorem suggests that chordless cycles satisfy the above two properties, making them suitable for representing circles in a graph.
To match circles, we introduce an augmented graph by inserting supernodes that embody these circles.
Given a length $L$, we can acquire corresponding chordless cycles whose length is no longer than $L$ for the query and target graphs as : $\gC_\gT=\{c | l(c)\le L,c\in CC_\gT\},\gC_\gQ=\{c | l(c)\le L, c \in CC_\gQ\}$ where $CC_\gT$ and $CC_\gQ$ are the chordless cycle set.
By setting $v_c$ as the supernode of any chordless circle $c$, we connect nodes from the circle $c$ to this supernode, resulting in an augmented graph. 
\eat{It is important to note that supernodes only compare against other supernodes to keep the matching of non-circles untainted.}
{Note that supernodes can only match other supernodes to keep the matching of non-circles untainted.}
To this end, we transform the circle matching as the subtree matching problem such that we can employ the proposed method on the augmented graph directly.
Unless otherwise stated, we keep all notations the same in the augmented graph to avoid abusing the notations.


\subsubsection{Dealing with Nodes' Attributes}
\label{sec:node_att}
Apart from circle structures, subgraph matching also involves node attributes.
Within the context of subgraph matching, learning with node attributes alone may be misleading because these cannot catch structural isomorphism. 
As a result, we employ the obtained subtree indicator matrix to supervise the learning process, aiming to filter out pairs that do not pass the test in the subtree matching.
We are thus motivated to enhance the node attributes by concatenating the subtree matching indicator, resulting in the node representation for the query and target graphs as follows:
\begin{equation}
    \begin{aligned}
    \footnotesize
\begin{cases}
\!\!\!&\!\!\!\!H^{(l+1)}_\gT=GNN^{(l)}_\gT(A_\gT,\concat(H_\gT^{(l)}, MLP(S^{(l)})))\\    \!\! \!&\!\!\!\!H^{(l+1)}_\gQ=GNN^{(l)}_\gQ(A_\gQ, \concat(H_\gQ^{(l)}, MLP(S^{(l)})^T))) 
\end{cases}
    \end{aligned}
\end{equation}
Here, we employ an MLP model to reduce the effect of the difference between the node attributes and the indicator matrix, where the latter behaves as a one-hot encoding feature.
We concatenate each pair of representations and then pass it to the MLP to obtain their similarity.
For the node pair $(i, j)$, we computer their similarity as  
\begin{equation}
\label{eq:gnn_s}
    \begin{aligned}
     \footnotesize
    [S^{(l+1)}_{gnn}]_{ij}=MLP(\concat([H_\gT^{(l)}]_{i},[H_\gQ^{(l)}]_{j})).
    \end{aligned}
\end{equation}
Now, we obtain a generalized indicator matrix that includes both the structure information and node attribute information 
\begin{equation}
    S^{(l+1)}=S_{gnn}^{(l+1)}\odot S_{subtree}^{(l+1)}
\end{equation}

\subsection{Implementation Details and Complexity Analysis}\label{sec:complexity}
We summarize the whole procedure in Algorithm~\ref{alg:algo} outlined in Appendix~\ref{app:code}.  We now analyze the time cost of our \ourmeth.  It is noted that our \ourmeth consists of two major parts: the subtree module and the GNN module.  Given $L$ layers and $K$ times of sampling, the complexity of the subtree module and the GNN module is $O(L*K*|V_\gT|*|E_\gQ|+L*|V_\gQ|*|E_\gT|)$ (according to the computation of lines 4-9 in Algo.~\ref{alg:algo}) and $O(L*|E_\gT|+L*|E_\gQ|+|V_\gT|*|V_\gQ|)$ (according to the computation of lines 10-12 in Algo.~\ref{alg:algo}), respectively.  
Since the query graph is often very small, we can treat $|V_\gQ|$ and $|E_\gQ|$ as constants. 
Therefore, the overall complexity is reduced to $O(|V_\gT|+|E_\gT|)$, attaining the linear time complexity. 
Please refer to Appendix~\ref{sec:implementations_app} for more details about the implementation.
{We also provide an empirical runtime comparison in Appendix~\ref{sec:runtime}.}

\section{Experiments}

\begin{table*}[!th]
\caption{\footnotesize Overall performance comparison in terms of accuracy.}
    \centering
    \footnotesize
 \scalebox{0.91}{
    \begin{tabular}{c|cccccccc}
    \toprule
            & Synthetic & Proteins & Mutag & Enzymes & Aids & IMDB-Binary & Cox2 & FirstMMDB\\
        \midrule
        SimGNN~\cite{Bai2019SimGNN} &  70.5$_{\pm\text{2.72}}$ & 96.2$_{\pm\text{0.97}}$ & 98.7$_{\pm\text{0.60}}$ & 98.6$_{\pm\text{1.08}}$ & 96.5$_{\pm\text{0.68}}$ & 85.0$_{\pm\text{19.58}}$ & 99.9$_{\pm\text{0.22}}$ & 82.40$_{\pm\text{0.17}}$ \\
       
        NeuroMatch~\cite{Rex2020NeuroMatch} & 65.7$_{\pm\text{8.98}}$ & 94.5$_{\pm\text{1.73}}$ & 99.2$_{\pm\text{0.22}}$ & 97.9$_{\pm\text{1.08}}$ & 97.4$_{\pm\text{0.96}}$ & 86.5$_{\pm\text{6.51}}$ & 100.0$_{\pm\text{0.00}}$ & 80.80$_{\pm\text{0.39}}$ \\
        IsoNet~\cite{Roy2022IsoNet} & 50.0$_{\pm\text{0.00}}$ & 60.0$_{\pm\text{10.02}}$ & 94.1$_{\pm\text{2.54}}$ & 91.0$_{\pm\text{7.78}}$ & 61.5$_{\pm\text{8.51}}$ & 83.1$_{\pm\text{3.69}}$ & 95.8$_{\pm\text{3.89}}$ & / \\
        GMN-embed~\cite{Li2019GMN} & 56.6$_{\pm\text{8.61}}$ & 93.8$_{\pm\text{2.41}}$ & 90.8$_{\pm\text{6.16}}$ & 89.4$_{\pm\text{16.44}}$ & 78.3$_{\pm\text{6.92}}$ & 69.3$_{\pm\text{15.18}}$ & 69.7$_{\pm\text{18.20}}$ &  69.1$_{\pm\text{30.29}}$ \\
        
        GraphSim~\cite{Bai2020GraphSim} & 50.0$_{\pm\text{0.00}}$ & 82.5$_{\pm\text{0.31}}$ & 89.5$_{\pm\text{2.59}}$ & 88.2$_{\pm\text{1.79}}$ & 75.6$_{\pm\text{6.53}}$ & 88.9$_{\pm\text{2.81}}$ & 95.5$_{\pm\text{0.94}}$ & 
        86.6$_{\pm\text{9.71}}$ \\
     
        GOT-Sim~\cite{Doan2021GOTSim} & 53.0$_{\pm\text{2.74}}$ & 57.2$_{\pm\text{8.52}}$ & 86.8$_{\pm\text{6.92}}$ & 68.7$_{\pm\text{14.15}}$ & 70.6$_{\pm\text{3.08}}$ & 81.3$_{\pm\text{14.60}}$ & 94.8$_{\pm\text{1.04}}$ & /  \\

        \midrule
        { \ourmeth} & {\bf 74.3$_{\pm\text{0.22}}$} & {\bf  100.0$_{\pm\text{0.00}}$} & {\bf  100.0$_{\pm\text{0.00}}$} & {\bf 99.9$_{\pm\text{0.22}}$ } & {\bf99.5$_{\pm\text{0.27}}$} & {\bf93.3$_{\pm\text{1.03}}$} & {\bf 100.00$_{\pm\text{0.00}}$} & {\bf 100.00$_{\pm\text{0.00}}$} \\
    
        \bottomrule
    \end{tabular}
    }
    \label{tab:overall}
\end{table*} 
\begin{table*}[!th]
\footnotesize
\caption{ Results of experimenting the uniformly distributed data in terms of accuracy. }
    \scalebox{0.9}{
    \centering
    \footnotesize
    
    \begin{tabular}{c|c|ccccccccc}
    \toprule
            & \cellcolor{Gray}Synthetic$^+$ & Synthetic& Proteins$^{*}$ & Mutag$^{*}$ & Enzymes$^{*}$ & Aids$^{*}$ & IMDB-Binary$^{*}$ & Cox2$^{*}$ & FirstMMDB$^{*}$ \\
        \midrule
              SimGNN & \cellcolor{Gray}84.1$_{\pm\text{3.40}}$ & 70.5$_{\pm\text{2.72}}$ & 64.6$_{\pm\text{3.36}}$ & 80.3$_{\pm\text{6.18}}$ & 76.0$_{\pm\text{2.12}}$ & 73.2$_{\pm\text{6.06}}$ &  72.0$_{\pm\text{2.45}}$ & 88.2$_{\pm\text{2.05}}$ & 53.2$_{\pm\text{6.61}}$   \\
       NeuroMatch & \cellcolor{Gray} 74.5$_{\pm\text{2.57}}$ &  65.7$_{\pm\text{8.98}}$ & 52.8$_{\pm\text{4.76}}$ & 90.4$_{\pm\text{2.88}}$ & 86.6$_{\pm\text{3.64}}$ & 75.6$_{\pm\text{17.78}}$ & 60.4$_{\pm\text{10.11}}$ & 91.0$_{\pm\text{5.70}}$ & 50.0$_{\pm\text{0.00}}$  \\
       \midrule
        \ourmeth & \cellcolor{Gray} \textbf{86.6}$_{\pm\text{1.44}}$ &  
        \textbf{74.3}$_{\pm\text{0.22}}$  &  \textbf{83.4}$_{\pm\text{2.97}}$ & \textbf{99.2}$_{\pm\text{0.84}}$ & \textbf{96.0}$_{\pm\text{2.16}}$ & \textbf{95.0}$_{\pm\text{1.41}}$ & \textbf{90.2}$_{\pm\text{1.79}}$ & \textbf{99.8}$_{\pm\text{0.45}}$ & \textbf{86.4}$_{\pm\text{7.44}}$ \\
        \bottomrule
    \end{tabular}
    }
    
    \label{tab:overal_abl_subtree}
\end{table*} 

\begin{figure*}[!ht]
	\centering
   \subfigure[\# layers  vs. accuracy]{\includegraphics[scale=0.33]{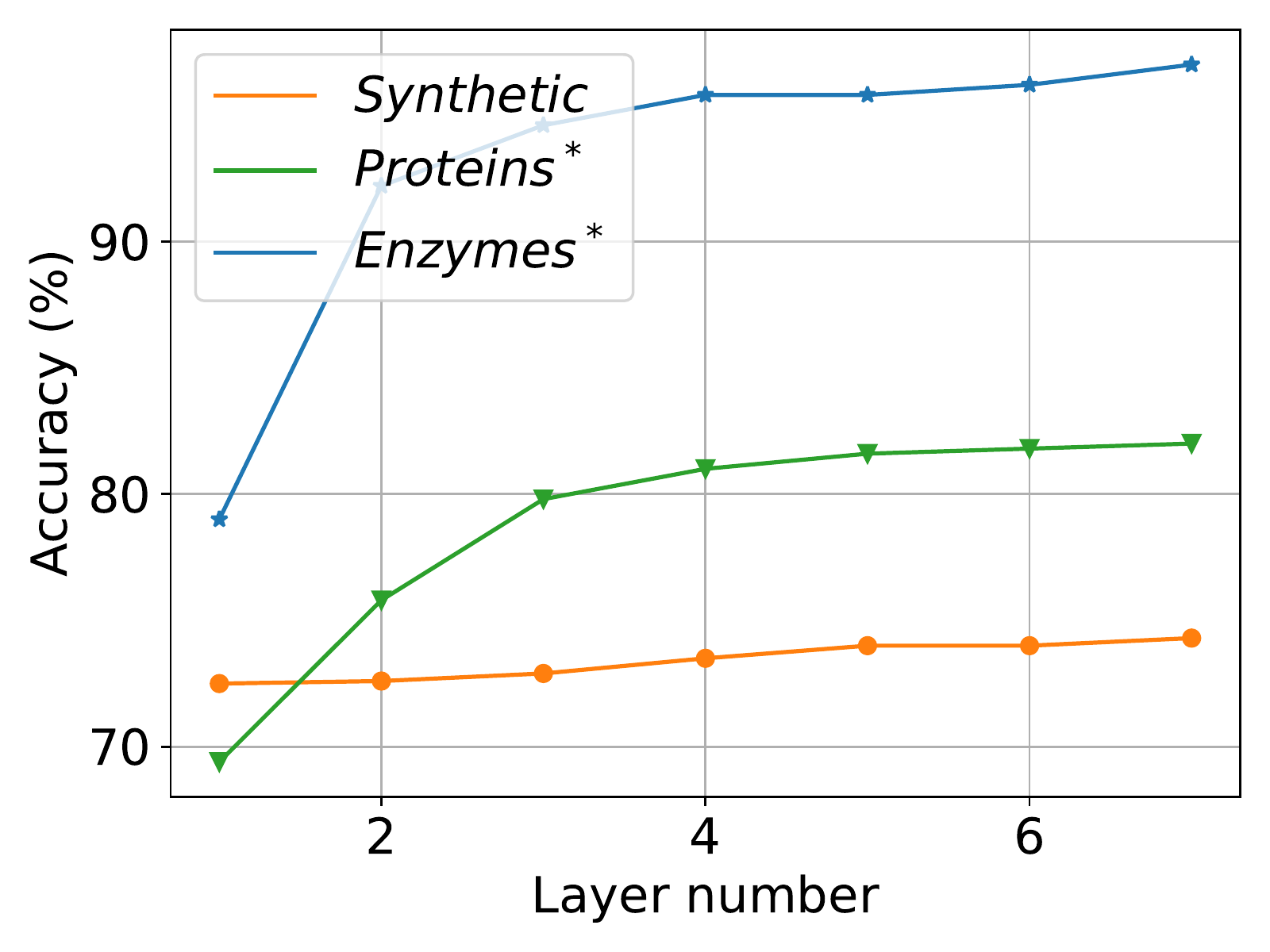}\label{fig:abs_layer}}
   \subfigure[\# sampling vs. accuracy]{\includegraphics[scale=0.33]{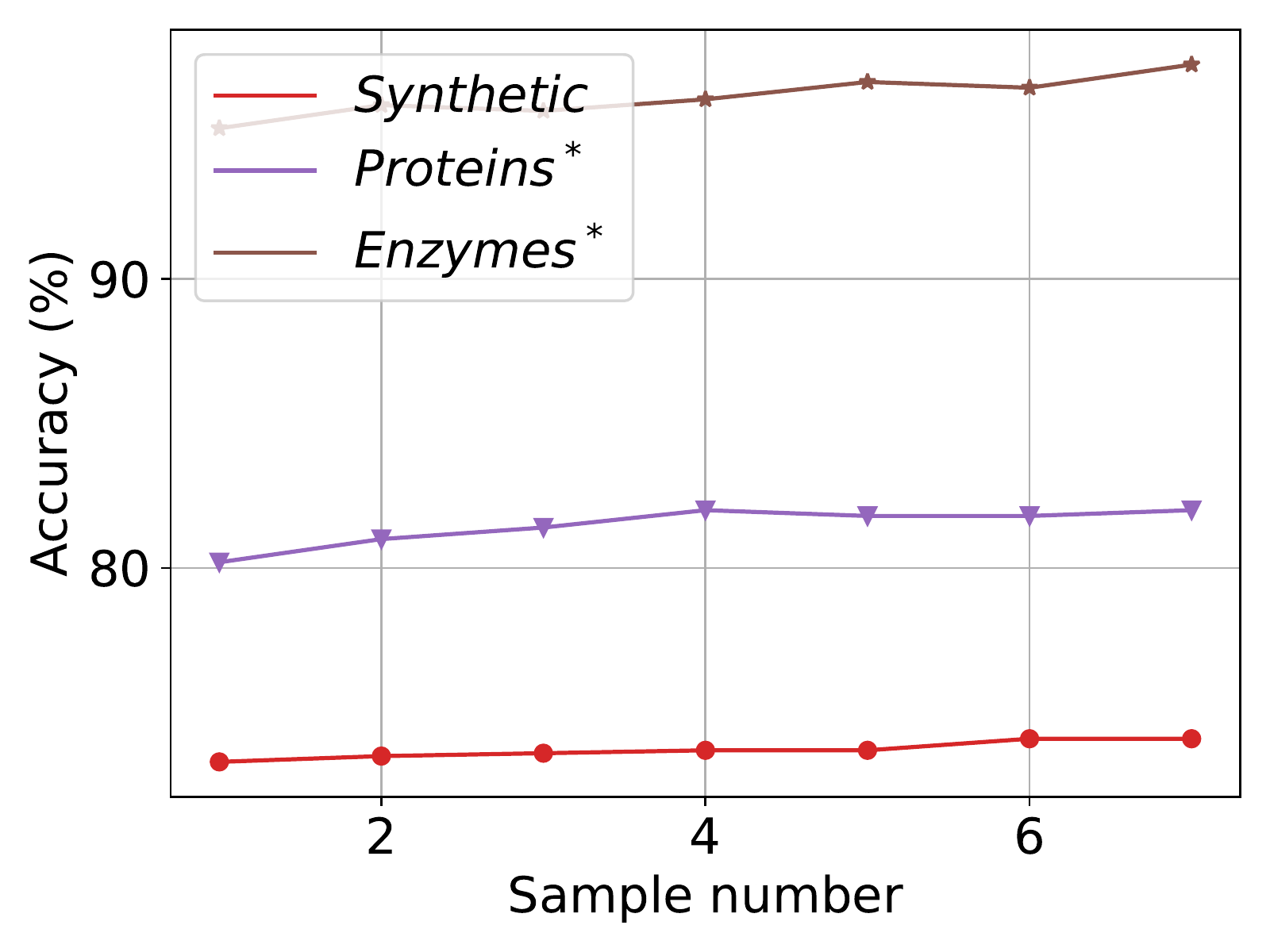}\label{fig:abs_samples}}
   \subfigure[The convergence comparison ]{\includegraphics[scale=0.33]{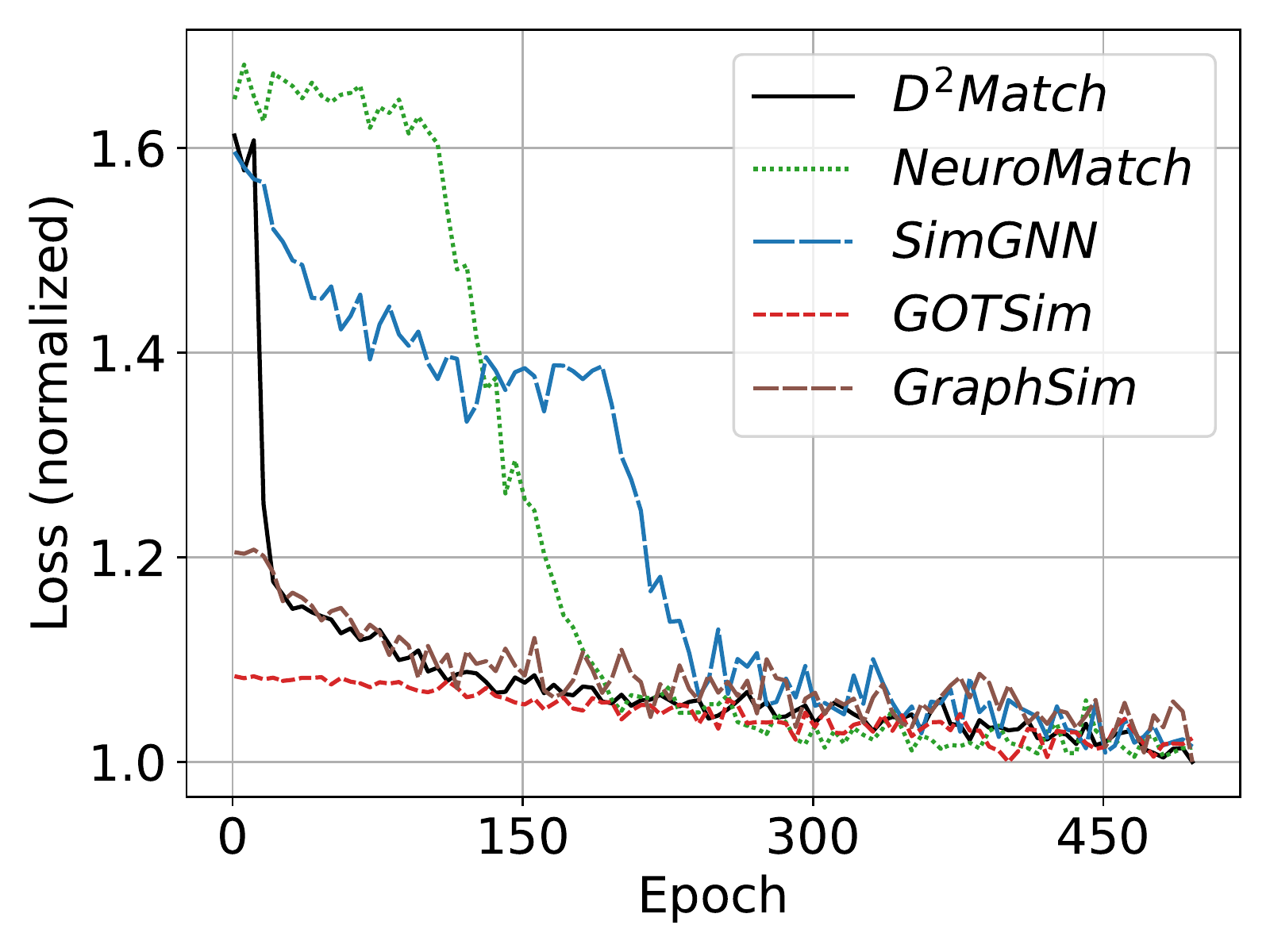}\label{fig:convergence}}
	\caption{\footnotesize  We conduct sensitivity analysis on our \ourmeth by varying the number of layers and sampling.  In Fig.~\ref{fig:convergence}, we present the convergence curve on our \ourmeth and four strong baselines.}
	\label{fig: ABLATION_CONVERGENCE}
	\vspace{-0.1in}
\end{figure*}

Here, we conduct extensive experiments to address the following questions: 
{(1)} \textit{ How does our \ourmeth perform comparing to SOTA GNN-based methods? } 
{(2)} \textit{ Why does our \ourmeth work better than other GNN-based methods?} 
{(3)} \textit{ How does our \ourmeth perform comparing to heuristic CO-based methods?} 
{(4)} \textit{ What is the effect of our \ourmeth on the hyperparameters?}
{(5)} \textit{ What is the convergence behavior of our \ourmeth? }
Sec.~\ref{sec:main_results}-Sec.~\ref{subsec:convergence} answer the above questions accordingly.

\subsection{Experimental Settings}
\label{sec:settings}
{\textbf{Datasets and Experimental Setup.}}
We implement our experiments on both synthetic and real-world datasets, which are collected from a large variety of applications.
We aim to obtain pairs of query and target graphs, along with labels indicating whether a query is isomorphic to the target.
We first generate synthetic data by utilizing ER-random graphs and WS-random graphs\eat{as in}~\citep{Rex2020NeuroMatch}. 
We keep edge densities the same in both positive and negative samples to ensure consistency in the distribution. 
This balance avoids potential biases during learning.
For the real-world data, we follow the setting in~\citep{Rex2020NeuroMatch}, including 
Cox2, Enzymes, Proteins, IMDB-Binary, MUTAG, Aids, and FirstMMDB. 
Please refer to Appendix \ref{subsec: more experiments} for additional experiments, such as Open Graph Benchmark datasets\citep{hu2020ogb}.


We employ these raw graphs as target graphs and generate the positive query graphs using random breadth-first search sampling from the target graphs.
The negative query graphs are randomly generated.
Similar to the synthetic data, we require the edge density in both positive and negative samples to be as close as possible. 
We split each dataset into training and testing at a ratio of $4:1$ and report the average classification accuracy under the five-fold cross-validation.  \eat{, where the ratios for training and testing are $80\%$ and $20\%$, respectively.}  

{\textbf{Baselines.}}
For a fair comparison, we select the following SOTA GNN-based competitors: 
SimGNN~\citep{Bai2019SimGNN}, 
NeuralMatch~\citep{Rex2020NeuroMatch}, 
IsoNet~\citep{Roy2022IsoNet}, 
GMN-embed~\citep{Li2019GMN},
GraphSim~\citep{Bai2020GraphSim}, and 
GOT-Sim~\citep{Doan2021GOTSim}.
These all incorporate graph neural networks into subgraph matching. 
For combinatorial optimization metheds, we choose GQL~\citep{he2008GraphQL}, QSI~\citep{Shang2008QuickSI}, CECI~\citep{Bhattarai2019CECI} and set-intersection method LFTJ~\citep{Sun2020HeuristicCompare}. These are traditional methods that can find exact solutions.

\subsection{Main Results}
\label{sec:main_results}
\eat{We summarize the overall performance in Table~\ref{tab:overall}. We report the best results of each model in the best possible settings, e.g., we run each model up to $500$ epochs.}

Table~\ref{tab:overall} reports the overall performance of all compared models, where each model achieves the best results in all possible settings up to $500$ epochs.  The accuracy of GOT-Sim and IsoNet on FirstMMDB is omitted due to exceeding time and memory.  We observe that
\begin{compactitem}[-]
\item For the synthetic dataset, the overall performance is much lower than that in the real-world datasets.  A reason is that the synthetic dataset is more complicated, e.g., with a higher edge density, than real-world datasets, which makes the matching more challenging.  
By examining more details, IsoNet, GMN-embed, GraphSim, and GOT-Sim attempt to employ a node-level assignment matrix to capture matching between graphs, which underestimates the importance of global structure.  They can only yield around $50\%$ accuracy.  SimGNN and NeuroMatch try to learn the global representation and attain the accuracy of $70.5\%$ and 65.7\%, respectively.  
\item For the real-world datasets, \ourmeth attains superior performance and beats all baselines.  Among seven real-world datasets, \ourmeth attains $100\%$ accuracy in four datasets, i.e., Protein, Mutag, Cox2, and FirstMMDB, while at least $99.5\%$ in Enzymes and Aids, and $93.3\%$ in IMDB-Binary.  
\item Overall, \ourmeth has explicitly modeled subtrees and consistently attained the best performance among all compared methods.  The promising results confirm our theoretical analysis. 

\end{compactitem} 

\vsa
\begin{table}[!h]
\footnotesize
\caption{ Time comparison with CO-based metheds(seconds). }
    \scalebox{1.0}{
    \centering
    \footnotesize
    
    \begin{tabular}{c|ccccc}
    \toprule
            & Proteins & Mutag & Enzymes & Aids & Cox2 \\
        \midrule
              GQL & 3.82  & 2.05  & 2.98 & 6.94 & 3.34 \\
              QSI & 10.37  & \textbf{0.16}  & 4.63 & 19.67 & 3.22 \\
              CECI & 19.74  & 0.27  & 13.74 & 19.15 & 6.19 \\
              LFTJ & 9.93  & 0.21  & 4.61 & 20.46 & 5.69 \\
       \midrule
              \ourmeth & \textbf{1.12}  & 1.04  & \textbf{1.18} & \textbf{1.88} & \textbf{1.68} \\
        \bottomrule
    \end{tabular}
    }
    \label{tab:co_runtime_compare}
\end{table} 
\vsa\vsa
\begin{table}[!h]
\footnotesize
\caption{ Efficiency comparison with CO-based metheds.  }
    \scalebox{1.0}{
    \centering
    \footnotesize
    
    \begin{tabular}{c|ccccc}
    \toprule
            & Proteins & Mutag & Enzymes & Aids & Cox2 \\
        \midrule
              GQL & 96.2  & 100.0  & \textbf{100.0} & \textbf{100.0} & 100.0 \\
              QSI & 96.8  & 100.0  & 99.8 & 94.8 & 100.0 \\
              CECI & 91.2  & 100.0  & 99.4 & 91.0 & 100.0 \\
              LFTJ & 87.0  & 100.0  & 89.8 & 95.8 & 100.0 \\
       \midrule
              \ourmeth & \textbf{100.0}  & \textbf{100.0}  & 99.9 & 99.5 & \textbf{100.0} \\
        \bottomrule
    \end{tabular}
    }
    \label{tab:co_efficiency_compare}
\end{table} 

\subsection{Benefit of Our Core Design: The Subtree Matching}

The following experiments verify how GNNs deployed in our \ourmeth function differently in existing GNN-based subgraph matching methods.  That is, our \ourmeth utilizes GNNs to explicitly model subtrees while existing methods optimize the graph representations via memorizing the data distribution divergence.  We additionally construct new datasets and denote them with $^*$.  The new datasets exclude the data distribution effect by the following steps: following the same way to generate the positive samples as in Sec.~\ref{sec:settings}, continuing to perform edge dropping and insertion on the clipped graphs together to obtain the final negative samples.  This strategy aims to make sure the generated samples follow almost the same distribution in terms of edges.  For the synthetic dataset, given that the positive and negative data in the original dataset are generated randomly and following the same distribution, we construct another synthetic dataset without following the same density distribution for better comparison.
Since SimGNN and NeuoMatch are the best-performing GNN-based methods in Table~\ref{tab:overall}, we compare our \ourmeth with them.  

Results in Table~\ref{tab:overal_abl_subtree} show that \ourmeth outperforms SimGNN and NeuoMatch by a much larger margin, achieving $2.5\%-33.2\%$ improvement.  
This phenomenon aligns with our hypothesis that the gain of other learning-based methods is distribution-dependent, which results in a significant performance drop on evenly-distributed data.
Moreover, the overall performance on Synthetic$^+$ is much better than that on Synthetic.  This implies that data following non-even distribution will make the matching much easier.  This is in line with the results in Table~\ref{tab:overall}, i.e., the GNN-based methods tend to capture the distribution divergence rather than performing matching.

\subsection{Comparison with exact combinatorial solutions}

Before detailing the comparison between \ourmeth and CO-based methods, we would like to emphasize that our method and CO-based heuristics are driven by different purposes and are therefore suitable for different scenarios. Heuristic algorithms can find exact solutions, but the worst-case time complexity is exponential, and they are generally suitable for scenarios with high matching accuracy requirements, such as graph databases. Our method is suitable for efficient machine learning tasks with time constraints.
To verify the advantages of our method, we compared it from three perspectives:

\textbf{Solution time:} Since heuristic algorithms can find exact solutions, the metric of accuracy is no longer distinguishable. Therefore, we mainly compared our algorithm and heuristic algorithms in terms of the average running time on the original dataset. We report the average running time of our method and heuristic algorithms on the test set, as shown in Table~\ref{tab:co_runtime_compare}. The results demonstrate that our method not only has a consistently stable running time, but it is also significantly faster than heuristic algorithms on most datasets, often by an order of magnitude.

\textbf{Solution efficiency:} To compare the efficiency of different solutions, we compared the success rates of heuristic solvers to the accuracy of our method, shown in Table~\ref{tab:co_efficiency_compare}. The success rate is defined as the proportion of samples that an algorithm completes within 10 seconds per sample on the original dataset. Considering the constraint of execution time, our method's accuracy and heuristic methods' success rate are comparable to those based on heuristic algorithms for most real datasets.

\textbf{Scalability:} 
To validate the scalability of our model, we conducted experiments using larger datasets. Specifically, we used the "firstmm" dataset consisting of 50,000 nodes, as well as the "DD" dataset consisting of 300,000 nodes and one million edges. We compared our method with these datasets in terms of runtime and success rate.
As shown in Table~\ref{tab:co_scalability_compare}, our method demonstrates better efficiency than heuristic algorithms on large-scale graphs. On the DD dataset, our method achieves an accuracy rate of nearly 100\%. 
On the firstmm dataset, our method has a relatively low accuracy rate, but it is comparable to some heuristic algorithms. These results suggest that our method is scalable.

\vsa
\begin{table}[!h]
\footnotesize
\caption{ Results of experimenting the scalability. }
    \scalebox{1.1}{
    \centering
    \footnotesize
    
    \begin{tabular}{c|cc|cc}
    \toprule
            & \multicolumn{2}{c|}{Time(seconds)} & \multicolumn{2}{c}{Success rate(\%)} \\
        \midrule
            & DD & Firstmm & DD &  Firstmm \\
        \midrule
              GQL & 58.49  & 149.15  & \textbf{100.0} & 60.0 \\
              QSI & 80.18  & 127.21  & 94.0 & 60.4 \\
              CECI & 128.93  & 163.73  & 48.4 & 69.6 \\
              LFTJ & 80.36  & 89.85  & 92.8 & \textbf{77.2} \\
       \midrule
              \ourmeth & \textbf{46.61}  & \textbf{32.62}  & 99.2 & 60.0 \\
        \bottomrule
    \end{tabular}
    }
    \label{tab:co_scalability_compare}
\end{table}

Furthermore, we included the success rates of our heuristic methods on synthetic datasets with varying numbers of nodes. This demonstrates the scalability of our method on synthetic data, as shown in Table~\ref{tab:co_syn_scalability_compare}.
In general, the success rate of heuristic methods decreases rapidly as the scale increases. However, our method has consistently maintained an accuracy of around 70\%.

\vsa
\begin{table}[!h]
\footnotesize
\caption{ Results of synthetic datasets. }
    \scalebox{1.1}{
    \centering
    \footnotesize
    
    \begin{tabular}{c|cccccc}
    \toprule
              Nodes & 20 & 40 & 60 & 80 & 100 & 200\\
        \midrule
              GQL & \textbf{75.0}  & \textbf{81.8}  & \textbf{78.0} & 63.2 & 56.8 & 32.0 \\
              QSI & 70.4  & 61.6  & 19.4 & 12.2 & 11.2 & 2.6 \\
              CECI & 66.0  & 51.2  & 13.8 & 12.4 & 15.0 & 16.8 \\
              LFTJ & 60.2  & 71.4  & 25.6 & 15.8 & 12.0 & 5.2 \\
       \midrule
              \ourmeth & 70.0  & 70.2  & 74.0 & \textbf{63.4} & \textbf{71.0} & \textbf{67.0} \\
        \bottomrule
    \end{tabular}
    }
    \label{tab:co_syn_scalability_compare}
\end{table}

\subsection{Sensitivity Analysis}

\textbf{Effect of $L$, the depth of a subtree.} 
We test the effect of the depth of a subtree, i.e., the number of the hidden layers, and change it from 1 to 7.  Results in Fig.~\ref{fig:abs_layer} show that \ourmeth reaches its best performance when the number of layers is 6.  \ourmeth only needs a few layers to achieve decent performance, suggesting it can scale up to large size graphs.

\textbf{Effect of $K$, the times of samples.}  
Intuitively, sampling more data to train the model will yield better performance.  We vary $K$ from 1 to 7 and show the results in Fig.~\ref{fig:abs_samples}.  Surprisingly, the results show that by sampling five times, we can obtain the best performance on all datasets.  This demonstrates that \ourmeth can attain decent performance at a low computation cost.

We also ablate \ourmeth with circles and node attributes at different settings, and all experiments show results consistent with our theoretical analysis. Please refer to Appendix~\ref{subsec:ablation} for more details.
 

\subsection{Convergence Analysis}
\label{subsec:convergence}
Figure~\ref{fig:convergence} provides the training loss of \ourmeth and four baselines on Synthetic, where we only select baselines with the same loss functions as ours, such as MSE or CE, for a fair comparison. 
The results show that (1) \ourmeth converges the fastest due to its power of explicitly modeling the subtrees. (2) NeuroMatch and SimGNN perform matching through learning graph-level representations, which need more epochs to converge for capturing the local structure. (3) GOTSim and GraphSim attain the lowest loss in the beginning but show the weakest convergence ability compared to others because they can only capture the node-level representations and fail to learn meaningful subgraph matching.  Consequently, they yield the worst performance as reported in Table~\ref{tab:overall}.


\section{Conclusion and Future work}
{
In this paper, we propose \ourmeth for subgraph matching, which degenerates the subgraph matching problem into perfect matching in a bipartite graph and proves that the matching procedure can be implemented via the built-in tree-structure aggregation on GNNs, which yields linear time complexity.  We also incorporate circle structures and node attributes to boost the matching performance.  Finally, we conduct extensive experiments to show that \ourmeth achieves significant improvement over competitive baselines and indeed exploits subtrees for the matching, which is different from existing GNN-based methods for memorizing the data distribution divergence. 

\ourmeth can be further explored in several promising directions.  First, we can investigate more degeneracy mechanisms to tackle more complicated graphs.  
Second, we can exploit other information, e.g., positional encoding,  to boost the model performance. 
Third, we can extend our \ourmeth to more real-world applications, e.g., document matching, to investigate its capacity.  
}

\section{Acknowledgement}

The work was fully supported by the IDEA Information and Super Computing Centre (ISCC) and was partially supported by the National Nature Science Foundation of China (No. 62201576), the National Key Research and Development Program of China (No. 2020YFB1708200),  the "Graph Neural Network Project" of Ping An Technology (Shenzhen) Co., Ltd. and AMiner.Shenzhen SciBrain fund.


\bibliography{icml2023}
\bibliographystyle{icml2023}

\newpage
\appendix
\onecolumn
\newpage
\section{Appendix}
\subsection{The Pseudo-Code of \ourmeth}\label{app:code}
The pseudo-code of \ourmeth is outlined as follows:
\begin{algorithm}[hb]
\footnotesize
    \caption{ The \ourmeth algorithm}\label{alg:algo}
\begin{algorithmic}[1]
\REQUIRE {Query graph $G_\gQ(V_\gQ,E_\gQ)$ with node attributes $X_\gQ$, target graph $G_\gT(V_\gT,E_\gT)$ with node attributes $X_\gT$, iteration number: $L$, sample number: $K$. }
\ENSURE{Is $G_\gQ$ isomorphic to $G_\gT$}
\STATE $G_\gQ(V_\gQ,E_\gQ)\gets ChordlessCycleAugment(G_\gQ)$; \\
$G_\gT(V_\gT,E_\gT)\gets ChordlessCycleAugment(G_\gT)$;
\STATE $H_\gQ^{(0)}= X_\gQ\in\mathbb{R}^{|V_\gQ|\times D}$; $H_\gT^{(0)}= X_\gT\in\mathbb{R}^{|V_\gT|\times D}$; 
\STATE $S^{(0)} = InitialAssignMatrix(X_\gT,X_\gQ) \in \mathbb{R}^{|V_\gT|\times |V_\gQ|}$
\FOR{$l = 0, 1 ..., L-1$}
  \FOR{$k = 0,1,...,K$ }
    \STATE 
    $\tilde{A}_\gQ^{(k)} = DropEdge(A_\gQ)\in\mathbb{R}^{|V_\gQ|\times |V_\gQ|}$; 
    \STATE Calculate $\Phi^{(l+1)}(\tilde{A}_\gQ^{(k)}, A_\gT)$ according to Eq.~(\ref{eq:Phi})
 \ENDFOR
    \STATE $ S^{(l+1)}_{subtree}=\bigodot_{k=0}^{K}{\Phi^{(l+1)}(\tilde{A}_\gQ^{(k)},A_\gT)}\in\mathbb{R}^{|V_\gT|\times|V_\gQ|}$;
    
    \STATE $H^{(l+1)}_\gT=GNN^{(l)}_\gT(A_\gT,concat[H_\gT^{(l)},MLP(S^{(l)})])\in\mathbb{R}^{|V_\gT|\times d}$; 
    \STATE $H^{(l+1)}_\gQ=GNN^{(l)}_\gQ(A_\gQ,concat[H_\gQ^{(l)},MLP((S^{(l)})^T)])\in\mathbb{R}^{|V_\gQ| \times d}$;
    \STATE Compute $S^{(l+1)}_{gnn} \in \mathbb{R}^{|V_\gT|\times |V_\gQ|}$  according to Eq.\ref{eq:gnn_s}
    \STATE $S^{(l+1)}=S_{gnn}^{(l+1)}\odot S_{subtree}^{(l+1)}$;
\ENDFOR 
\STATE $result = CheckAssign(S^{(L)})$
\end{algorithmic}
\end{algorithm}

\subsection{Main Results}
\begin{theorem}
Given a target graph $G_\gT(V_\gT,E_\gT)$ and a query graph $G_\gQ(V_\gQ,E_\gQ)$, 
if $G_\gQ \subset G_\gT$, and the subtree generation function $\Psi$ as defined in Eq.~(\ref{eq:tree_generation}) meets the following  condition: 
\begin{equation}
\eat{\forall graph\ (G_\gS, G_\g), if\  G_\gS\subset G_\g, then\ \Psi(G_\g)\subset \Psi(G_\gS),}
\forall \ \mbox{graph pair} \ (G_\gS, G), \mbox{if}\  G_\gS\subset G, \mbox{then}\ \Psi(G_\gS)\subset \Psi(G),
\end{equation}
then the subgraph isomorphic mapping  $\xi\!:\!V_\gQ\to V_\gT$ ensures the $l$-hop subtrees of the subgraph is isomorphic to the subtrees of the corresponding subgraph: 
\begin{equation}
         \forall l\ge1,\ervq\in V_\gQ,  \ervt = \xi(\ervq)\in V_\gT \Rightarrow  T_\ervq^{(l)}\subset T_{\ervt}^{(l)},  
\end{equation}
\label{theo:necessary_app}
\end{theorem}
\begin{proof}
According to the definition of subgraph matching~\citep{McCreesh2018SubIso},  when $G_\gQ$ is a subgraph of $G_\gT$, there must exists an injective function $\xi$ : $V_\gQ\to V_\gT$, such that $\forall \ervq_i, \ervq_j \in V_\gQ, (\ervq_i, \ervq_j) \in E_\gQ \Rightarrow (\xi(\ervq_i),\xi(\ervq_j)) \in E_\gT$. 
For any subgraph in the query graph, e.g., $S(V_S,E_S)\in G_\gQ$, we always have a subgraph in the original graph $G_\gT$, denoted as $G_S (V_G,E_G)$, that corresponds to the set of the query node as $V_G=\xi(V_S)$. This tells us that $S\subset G_S$. 
{According to this,} consider any given node from $V_\gQ$: $\ervq\in V_\gQ$, $S_\ervq^{(l)}$ is a subgraph of $G_\gQ$ and its image $G_{S_\ervq^{(l)}}$ in $G_\gT$, i.e. $ S_\ervq^{(l)} \subset G_{S_\ervq^{(l)}} $.
By definition, the node in $S_\ervq^{(l)}$ or $G_{S_\ervq^{(l)}}$ is at most $l$-hop from node $\ervq$ or $\ervt = \xi(\ervq)$, we know that $G_{S_\ervq^{(l)}}$ must be a subgraph of $S_{\ervt}^{(l)}$, i.e.,$G_{S_\ervq^{(l)}} \subset S_{\ervt}^{(l)}$.
Put all together, we have  $S_\ervq^{(l)} \subset G_{S_\ervq^{(l)}} \subset S_{\ervt}^{(l)}$. Based on the listed constrain, we then have $T_\ervq^{(l)}\subset T_{\ervt}^{(l)}$.

\end{proof}

\eat{
\begin{theorem}
Given a node $b$ in the query graph and a node $a$ in the target graph, the following three conditions are equivalent:\\ 
1) $T_b^{(l+1)}\subset T_a^{(l+1)}$;
\\
2) There exists an injective function on the neighborhood of these nodes as $f:N(b)\to N(a)$, \st $\forall b_i\in N(b), a_i=f(b_i), T_{b_i}^{(l)}\subset T_{a_i}^{(l)}$. \\
3) There exists a perfect matching on the bipartite graph $B(V_{ab}, E_{ab})$, where $V_{ab} = N(a) \cup N(b)$ and $\forall a_i\in N(a), b_i\in N(b), (a_i, b_i)\in E_{ab}$ if and only if $T_{b_i}^{(l)}\subset T_{a_i}^{(l)}$.
\label{theo:recursive_equal_app}
\end{theorem}
\todo{make this same in the main paper}
}
\begin{theorem}
Given a node $\ervq$ in the query graph and a node $\ervt$ in the target graph, the following three conditions are equivalent:
\begin{compactenum}[1)]
\item $T_\ervq^{(l+1)}\subset T_\ervt^{(l+1)}$.
\item There exists an injective function on the neighborhood of these nodes as $f\!:\!N(\ervq)\to N(\ervt)$, \mbox{s.t.} $\forall \ervq_i\in N(\ervq), \ervt_i=f(\ervq_i), T_{\ervq_i}^{(l)}\subset T_{\ervt_i}^{(l)}$. 
\item  There exists a perfect matching on the bipartite graph $B^{(l)}(N(\ervt), N(\ervq), E)$, where $\forall \ervt_j\in N(\ervt), \ervq_i\in N(\ervq), (\ervt_j, \ervq_i)\in E$ if and only if $T_{\ervq_i}^{(l)}\subset T_{\ervt_j}^{(l)}$.
\end{compactenum}
\label{theo:recursive_equal_app}
\end{theorem}
We prove this theorem by introducing the following two theorem. Theorem~\ref{theo:recursive_equal} shows that condition 1) is equivalent to condition 2), i.e. the WL subtree isomorphism test can be accomplished in a recursive manner then prove Theorem.~\ref{theo:bipartite} 
that the condition 2) equals to condition 3) which means every iteration in the recursive process equals to examine the existence of a perfect matching, respectively. \eat{\todo{This part is unreadable.}}

\begin{theorem}
Given a node $\ervq$ in the query graph and a node $\ervt$ in the target graph, the following two conditions are equal:\\ 
1) $T_\ervq^{(l+1)}\subset T_\ervt^{(l+1)}$, where $l$ is an integer and  $l\ge 1$.\\
2) There exists an injective function on the neighboring set of these nodes as $f\!:\!N(\ervq)\to N(\ervt), \mbox{s.t.} \forall \ervq_i\in N(\ervq), \ervt_i=f(\ervq_i), T_{\ervq_i}^{(l)}\subset T_{\ervt_i}^{(l)}$. 
\label{theo:recursive_equal}
\end{theorem}

\eat{\todo{Rewrite the notation to make it consistent with the update one}}
\begin{proof}
We assume $f_\ervq$ is a subtree isomorphism injective function in the condition 1), that $\forall$ node $u,v \in T_\ervq^{(l+1)}, (u,v)$ is an edge of $T_\ervq^{(l+1)} \Rightarrow ((f_\ervq(u),f_\ervq(v))$ is an edge of $T_\ervt^{(l+1)}$. Similarly We also assume $f_{\ervq_i}$ is subtree isomorphism injective in the condition 2).\\
On the one hand, if condition 1) is true then $f_\ervq$ exists.  
Using the property of WL tree, we have {$\forall \ervq_i \in N(\ervq), T_{\ervq_i}^{(l)}\subset T_\ervq^{(l+1)}$}, which means the $l$-order WL tree of any node $\ervq_i$ in $\ervq$'s neighbourhood belongs to the $l+1$-order WL tree originate from the node $\ervq$. 
This suggests that $f_\ervq$ maps $T_{\ervq_i}^{(l)}$ into a tree $T_{f(\ervq_i)}^{(l)} = T_{\ervt_i}^{(l)}$, which is a subtree of $T_{\ervt}^{(l+1)}$,. 
\eat{Let $a_i=f_b(b_i)$,}Then the condition 2) is true. \\
On the other hand, if condition 2) holds, then we define the mapping as $ f_\ervq(v)= \begin{cases}f_{\ervq_i}(v), & v \in T_{\ervq_i}^{(l)} \\ \ervq, & v=\ervq\end{cases}$.
Here, the function $f_\ervq(v)$ is a well-defined injective function because all $T_{\ervq_i}^{(l)}$ has no intersection.
This implies this is a subtree isomorphic mapping, so 1) holds.

\end{proof}

\eat{
Please refer the Apendix~\todo{} for the proof.
Intuitively, we can embody the above Condition 2 by maintaining an indicator matrix $S^{(l)}\in R^{|V_\gT|\times |V_\gQ|}$, where $S_{i j}^{(l)}= \begin{cases}1, & T_j^{(l)} \subset T_i^{(l)} \\ 0, & \text { else }\end{cases}$. 
This captures the relation between all pairs of nodes and thus can be used for matching.
Optimizing perfect matching is done by updating the matrix $S^{(l)}$ recursively.
Our theoretical analysis shows that the above condition 2 can be implemented as a perfect matching problem, i.e., what makes condition 2) true is equivalent to finding a perfect matching on a bipartite graph, as shown in the following theorem:}

The above theorem provides a recursive solution to the WL subtree isomorphism algorithm. Intuitively, we can maintain an indicator matrix $S^{(l)}\in R^{|V_\gT|\times |V_\gQ|}$, where $S_{\ervt \ervq}^{(l)}= \begin{cases}1, & T_\ervq^{(l)} \subset T_\ervt^{(l)} \\ 0, & \text { else }\end{cases}$. 
This matrix captures the relation between all pairs of nodes and thus can be used for recursion update.
Next, we will show that the update process can be implemented as a perfect matching problem, i.e., what makes condition 2) true is equivalent to finding a perfect matching on a bipartite graph, as shown in the following theorem:

\eat{\todo{Rest: whether need to change a, b to $\ervq$ and $\ervt$, to minimize the reviewers' memory}}
\begin{theorem}
Assume the neighboring set of node $\ervt$ and $\ervq$ as $X=N(\ervt)$ and $Y=N(\ervq)$, respectively. Accordingly, we form a bipartite graph as $B_{\ervt,\ervq}^{(l)}(X, Y, E)$. Here, we define the edges as  $E=\{(\ervt_i,\ervq_j):T_{\ervq_i}^{(l)}\subset T_{\ervt_j}^{(l)}  \}$, where $\ervt_i$ and $\ervq_j$ represent the $i$th and $j$th neighbour of node $\ervt$ and $\ervq$, respectively. Under this setting,  
 the injective function $f$ from the condition 2) in Theorem.~\ref{theo:recursive_equal} induces a perfect matching. 
\label{theo:bipartite}
\end{theorem} 

\begin{proof}
The injective function $f$ of condition 2) in Theorem~\ref{theo:recursive_equal} maps every node $\ervq_i$ in $N(\ervq)$ to $\ervt_i = f(\ervq_i) \in N(\ervt)$ and $T_{\ervq_i}^{(l)}\subset T_{\ervt_i}^{(l)}$ holds. While $T_{\ervq_i}^{(l)}\subset T_{\ervt_i}^{(l)}$ means $(\ervq_i, \ervt_i) \in E$, the injective $f$ naturally corresponds every node $\ervq_i$ to an edge $(\ervq_i,\ervt_j)$. Since $f$ is an injective function, $\ervq_{i_1} \neq \ervq_{i_2} \Rightarrow \ervt_{i_1} \neq \ervt_{i_2}$, indicating that all these edges $(\ervq_{i}, \ervt_{i}), i=1,...,|N(\ervq)|$ are different, which actually forms a perfect matching.
\end{proof}

\eat{\todo{miss proof}}

\eat{
\begin{theorem}
 Given a target graph and sampled query graph, we have their adjacency matrices as $A_\gT $ and $ \tilde{A}_\gQ$, and have the degree matrix of the sampled query graph as $\tilde{D}_\gQ=\diag[\sum_{s}((\tilde{A}_\gQ)_{:s})]$. 
 To perform the comparison of $|N(W)|\ge |W|$, we can formulate the comparison matrix $\Phi $ as $  \Phi =( Z_{N(W)}\ge 1) $, 
 where $Z_{N(W)}= \aggregate_{\scriptsize\mbox{sum}}(A_\gT,Z_W)$ and $  Z_W=\aggregate_{\max}(\tilde{D}_\gQ^{-1}\cdot\tilde{A}_\gQ,(S^{(l)})^T).$
\eat{ $\Phi =( Z_{N(W)}\ge 1)$, where $ Z_{N(W)}= \aggregate_{\scriptsize\mbox{sum}}(A_\gT,Z_W) $, and $ Z_W=\aggregate_{\max}(\tilde{D}_\gQ^{-1}\cdot\tilde{A}_\gQ,(S^{(l)})^T)$.
}
~\label{theo:znw_app}
\end{theorem}
}
\begin{theorem}\label{theo:znw_app}
Given the sampled query graph and the target graph, we can construct their adjacency matrices , $\tilde{A}_\gQ$ and $A_\gT$, and the degree matrix of the sampled query graph $\tilde{D}_\gQ=\mbox{diag}(\sum_{s}((\tilde{A}_\gQ)_{:s}))$. Here, we denote the indicator matrix at the $l$-th hop as $S^{(l)}$.  To check the validity of $|N(W)|\ge |W|$, we can check whether each element of $\Phi$ is true or not, where $\Phi \coloneqq Z_{N(W)}\ge 1 $, $Z_{N(W)}= \aggregate_{\scriptsize\mbox{sum}}(A_\gT,Z_W^T)$ and $  Z_W=\aggregate_{\max}(\tilde{D}_\gQ^{-1}\cdot\tilde{A}_\gQ,(S^{(l)})^T).$

\end{theorem}

\begin{proof}
For each node pair $\ervt,\ervq$ and their corresponding $W=N'(\ervq)$ in the sampled query graph, We first transform the neighboring set of $W$, i.e., $N(W)$, as following:
\begin{equation}
    \begin{aligned}
        & N(W)=\{\ervt_i \in N(\ervt) | \exists \ervq_j \in W= N'(\ervq), \mbox{s.t.} T_{\ervq_j}^{(l)}\subset T_{\ervt_i}^{(l)}\} \\     
        & = \{\ervt_i\in N(\ervt)|\exists \ervq_j\in W= N'(\ervq), \mbox{s.t.} S_{\ervt_i,\ervq_j}=1\} \\       
        & =\{\ervt_i\in N(\ervt)|\max_{\ervq'\in N'(\ervq)}{S_{\ervt_i,\ervq'}^{(l)}=1}\} \\                
        & =N(\ervt)\cap\{\ervt_i|\max_{\ervq'\in N'(\ervq)}{S_{\ervt_i,\ervq'}^{(l)}=1}\} \\
        & = N(\ervt)\cap M(\ervq)
    \end{aligned} 
\end{equation}
Let $M(\ervq)=\{\ervt_i|\max_{\ervq'\in N'(\ervq)}{S_{\ervt_i,\ervq'}^{(l)}=1}\}$,  
we can compute $M(\ervq)$ via a standard maximizing aggregation process on the sampled adjacency matrix $\tilde{A}_\gQ$, in which treats the indicator matrix $(S^{(l)})^T\in R^{|V_\gQ|\times |V_\gT|}$ as node attributes. 
This process will output the representation of node $\ervq$ as follows, 
\begin{equation}
    \begin{aligned}
        z_{\ervq,:} = max\{(S^{(l)})^T_{j,:},\forall j\in N'(\ervq)\},
    \end{aligned}
\end{equation}
The obtained vector $ z_{\ervq,:}$ is to represent the neighbours of node $\ervq$, i.e., $ M(\ervq) $, where $z_{\ervq i}= \begin{cases}1, & i \in M(\ervq) \\ 0, & \text { else }\end{cases}$. We rewrite this into a matrix format as
\begin{equation}
    \begin{aligned}
        Z_W=\aggregate_{\max}(\tilde{A}_\gQ,(S^{(l)})^T)
    \end{aligned}
\end{equation}
where $Z_W\in R^{|V_\gQ|\times |V_\gT|}$ and its $\ervq$-th row vector is $z_{\ervq:}$. 


Recall that we demand $N(W)=N(\ervt)\cap M(\ervq)$. After acquiring $M(\ervq)$, we can compute the $|N(W)|$ as follows,
\begin{equation}
    \begin{cases}
 & |M(\ervq)|=\sum_{i}z_{\ervq,i}\\
 & |N(W)|=\sum_{i}{z_{\ervq,i}}, i\in N(\ervt)
\end{cases}
\end{equation}
In essence, this is to implement a summation aggregation on the target graph using the node representation $Z_W$, i.e.,
\begin{equation}
    \begin{aligned}
        Z_{N(W)}= \aggregate_{\scriptsize\mbox{sum}}(A_\gT,Z_W^T) 
    \end{aligned}
\end{equation}
where $Z_{N(W)}\in R^{|V_\gT|\times|V_\gQ|}$ is an integer matrix and its element $(\ervt,\ervq)$ shows the score of $|N(W)|$ between node $\ervt$ and $\ervq$. 
This transformation converts the counting operation as aggregation such that we can 
check the aggregated values to determine whether there is a perfect matching.
Given a node pair $(\ervt, \ervq)$, we have $|N(W)| = [Z_{N(W)}]_{\ervt\ervq}$ and $|W|=|N'(\ervq)| = \sum_{s}[\tilde{A}_\gQ]_{\ervq s}$.  Therefore, the question becomes to check whether $[Z_{N(W)}]_{\ervt \ervq} \ge \sum_{s}[\tilde{A}_\gQ]_{\ervq s}$ holds.  We can then derive the perfect matching as follows: \
\begin{equation}
    \begin{aligned}
        & [Z_{N(W)}]_{\ervt \ervq}\ge \sum_{s}[\tilde{A}_\gQ]_{\ervq s} \\
        \Leftrightarrow & [Z_{N(W)}]_{\ervt \ervq}/\sum_{s}[\tilde{A}_\gQ]_{\ervq s} \ge 1  \\  
        \Leftrightarrow & [\aggregate_{\scriptsize\mbox{sum}}(A_\gT,Z_W^T)]_{\ervt \ervq}/\tilde{d}_\ervq\ge 1  \\
        \Leftrightarrow&[\aggregate_{\scriptsize\mbox{sum}}(A_\gT,Z_W^T)\cdot \tilde{D}_\gQ^{-1}]_{\ervt \ervq} \ge 1 \\
        \Leftrightarrow&[\aggregate_{\scriptsize\mbox{sum}}(A_\gT,Z_W^T\cdot \tilde{D}_\gQ^{-1})]_{\ervt \ervq} \ge 1 \\
    \end{aligned}
    \label{eq:agg_sum}
\end{equation}
where $\tilde{d}_\ervq$ is the degree of node $\ervq$ in the sampled graph.  The degree matrix of the sample graph is defined as $\tilde{D}_\gQ=\diag[\sum_{s}((\tilde{A}_\gQ)_{:s})]$.
Now recall that $\Phi$ is the matrix whose $(\ervt,\ervq)$ element is the comparison result of $|N(W)|$ and $|W|$ of $(\ervt,\ervq)$, according to eq \ref{eq:agg_sum}, we have:
\begin{equation}\label{eq:phi}
\Phi^{(l+1)}(\tilde{A}_\gQ, A_\gT)= \aggregate_{\scriptsize\mbox{sum}}(A_\gT,Z_W^T \cdot \tilde{D}_\gQ^{-1}) \ge 1,
\end{equation}
where 
\begin{equation}
\begin{aligned}
        Z_W^T \cdot\tilde{D}_\gQ^{-1}&=[\aggregate_{\max}(\tilde{A}_\gQ,(S^{(l)}))]^T\cdot\tilde{D}_\gQ^{-1} \\
        &=[\tilde{D}_\gQ^{-1} \cdot \aggregate_{\max}(\tilde{A}_\gQ,(S^{(l)}))]^T \\
        &=[\aggregate_{\max}(\tilde{D}_\gQ^{-1}\cdot\tilde{A}_\gQ,(S^{(l)}))]^T
\end{aligned}
\end{equation}

\end{proof}

\begin{theorem}
\label{theo:circles}
Every chordless cycle is atomic. Every chordless cycle $\gC_\gQ$ in an induced subgraph $G_\gQ$ must correspond to a chordless cycle $\gC_\gT$ in the origin graph $G_\gT$. 
\end{theorem} \vsa 

\begin{proof}
Chordless cycle does not have any chord, thus there is no smaller cycle in the chordless cycle, which means chordless cycle is atomic. Assuming $G_\gQ$ is a subgraph of $G_\gT$, every node of $\gC_\gQ$ must correspond to a node in $G_\gT$, and these nodes form a circle $\gC_\gT$ in $G_\gT$. 
Since $G_\gQ$ is an induced subgraph of $G_\gT$, if $\gC_\gT$ has a chord, then $\gC_\gQ$ must have a chord,  which
contradicts the condition that $\gC_\gQ$ is a chordless graph. 
\end{proof}

\subsection{Implementation Details}
\label{sec:implementations_app}

In this section we present the implementation details of our \ourmeth\footnote{The python implementation of \ourmeth will be available at https://github.com/XuanzhouLiu/D2Match-ICML23}. At the beginning of subtree isomorphism test, the model needs an initial indicator matrix $S_{subtree}^{(0)}$ as the input of the first iteration. 
According to the definition of the indicator matrix, $S_{subtree}^{(0)}$ shows the isomorphism relation between 
the subtree of $0$-hop neighbors,
which are  the nodes themselves in this case. 
Since all nodes will be isomorphic to each other if not considering the node attributes, the indicator matrix $S_{subtree}^{(0)}$ is actually a similarity matrix w.r.t node attributes. 
To get a similarity matrix of attributes, we can either directly calculate the similarity between nodes or employ neural networks on these attributes to learn the matrix. 
In our model, we implement both methods to initialize the matrix, called the initialization of the raw and the learnable:
\begin{equation}
\begin{aligned}
    Raw: S_{subtree}^{(0)} &= CosineSimilarity(X_\gT, X_\gQ) = Norm(X_\gT) \cdot Norm(X_\gQ^T) \\
    Learnable: S_{subtree}^{(0)} &= MLP_{}(X_\gT) \cdot MLP_{}(X_\gQ)^T
\end{aligned}
\end{equation}
where the raw initialization is to calculate the cosine similarity between the nodes' attributes, and the learnable initialization employs a MLP to generate hidden representations of nodes and compute their dot similarities.

In practice, we find the raw initialization performs better. 
This is because the node attributes of datasets are usually binary categorical vectors, which induces clear identification information of the nodes and can be easily captured by cosine similarity.

Our implementation of the GNN block in the model is slightly different from the description. Specifically, we use compute the similarity of each pair of nodes as:
\begin{equation}
    \begin{aligned}
     \footnotesize
        [S^{(l+1)}_{gnn}]_{ij}=MLP(concat([H_\gT^{(l)}]_{i},[H_\gQ^{(l)}]_{j})).
    \end{aligned}
\end{equation}
The main difference is that we do not output a $|V_\gT|\times |V_\gQ|$ matrix, but a $|V_\gT| \times |V_\gQ| \times |D^{(l+1)}|$ tensor, where $D^{(l+1)}$ denotes the hidden dim of $l+1$ layer. 
The intuition is that a tensor that represents the node pairs' similarity with vectors can retain more information than a similarity matrix with scalar elements. 
In this setting, the final indicator matrix $S^{(l+1)}$ can not be generated as $S^{(l+1)} = S_{gnn}^{(l+1)} \odot S_{subtree}^{(l+1)}$, because $S_{subtree}^{(l+1)} \in R^{|V_\gT| \times |V_\gQ|}$
but $S_{gnn}^{(l+1)} \in R^{|V_\gT| \times |V_\gQ| \times |D^{(l+1)}|}$. 
Thus we broadcast $S_{subtree}^{(l+1)}$ to $\tilde{S}_{subtree}^{(l+1)}$ where  $\forall k\in [0, D^{(l+1)}),[\tilde{S}_{subtree}^{(l+1)}]_{ijk}=[S_{subtree}^{(l+1)}]_{ij}$ and the final indicator matrix $S^{(l+1)} = S_{gnn}^{(l+1)} \odot \tilde{S}_{subtree}^{(l+1)}$

At the end of our models, we get the subtree indicator matrix $S_{subtree}^{(L)}$ and the GNN indicator matrix $S_{gnn}^{(L)}$. The model will output the final score from  $S_{subtree}^{(L)}$ and $S_{gnn}^{(L)}$, respectively. For the subtree module, we check whether the indicator matrix is feasible to induce the subgraph isomorphism. Note that for a node $i$ in the target graph and a node $j$ in the query graph, $i$ is possible to match $j$ unless $[S_{subtree}^{(L)}]_{ij} = 1$. 
So we check whether the subtree indicator matrix meets the following two conditions:
\begin{compactenum}[1)]
\item Every node in a query graph should match at least one node in the target graph:
\begin{equation}
    \begin{aligned}
     \forall j, &\max_{i}(S_{subtree}^{(L)})_{ij}=1 \\
     \Leftrightarrow &\sum_{j}\max_{i}(S_{subtree}^{(L)})_{ij}=|V_\gQ| \\
     \Leftrightarrow &\sum_{j}\max_{i}(S_{subtree}^{(L)})_{ij}/|V_\gQ|=1
    \end{aligned}
\end{equation}

\item The number of nodes in the target graph that match at least one node in the query graph is more than the number of nodes of query graph:
\begin{equation}
    \begin{aligned}
        &\sum_{i}\max_{j}(S_{subtree}^{(L)})_{ij}\ge |V_\gQ|
        \Leftrightarrow &\sum_{i}\max_{j}(S_{subtree}^{(L)})_{ij}/|V_\gQ| \ge 1
    \end{aligned}
\end{equation}
\end{compactenum}

To make the subtree model differentiable, we use a sigmoid to replace all the logical judgment in the model:
\begin{equation}
    \sigma(x) = \mathrm{Sigmoid}( ax + b)
\end{equation}
where $a,b$ are learnable parameters; $\sigma$ is the sigmoid function.
The result of subtree module can be fomulated as:
\begin{equation}
    \begin{aligned}
        r_{subtree} = \sigma(\sum_{i}\max_{j}(S_{subtree}^{(L)})_{ij}/|V_\gQ|) \cdot \sigma(\sum_{j}\max_{i}(S_{subtree}^{(L)})_{ij}/|V_\gQ|)
    \end{aligned}
\end{equation}

For the GNN module, we employ the neural tensor network(NTN)~\citep{Bai2019SimGNN} and generate a score according to the output of NTN and the aggregated indicator tensor:
\begin{equation}
    \begin{aligned}
        r_{gnn} = \sigma(MLP(concat[NTN(H_\gT^{(L)},H_\gQ^{(L)}), \sum_i\sum_j{S_{subtree}^{(L)}}]))
    \end{aligned}
\end{equation}
Where $H_\gT^{(L)}, H_\gQ^{(L)}$ are the node representations generated by the GNNs. $NTN$ is the NTN layer.

The final prediction is:
\begin{equation}
    r = r_{gnn} \cdot r_{subtree}
\end{equation}





Although the model's prediction is obtained by integrating the two modules, we can not directly train the model through the final score $r$ because it will bring difficulties in the training process. When fitting a negative sample, the resulting subtree module tends to be zero, forcing the overall gradient to be zero which hinders the training of the GNN module. 

Therefore, we train the two blocks with different objectives. For the subtree module which aims to learn the isomorphism relation, the result should be either 0 for not matching or 1 for matching. So we employ MAE loss to enforce the results to be either $0$ or $1$. 
For the GNN module, we use MSE to encourage the output of GNNs to capture the similarity.
Suppose the ground-truth label is $y$, and our loss function is
\begin{equation}
    L = MSE(r_{gnn}, y) + MAE(r_{subtree}, y)
\end{equation}

Both our model and all baselines use the Adam as optimizer and set the learning rate to $3e-4$. To ensure fairness, we set all models with adjustable number of layers to 5 layers, and set the hidden dimension to 10 to avoid overfitting.

\eat{
Therefore, we employ the MAE and MSE as the loss functions for training the subtree and GNN  modules, respectively. 
To satisfy the matching requirement, we use MAE to enforce the results to be either $0$ or $1$.}

\begin{figure*}[!ht]
    
	\centering
    \begin{minipage}[b]{0.3\linewidth}
		\centering
		\includegraphics[scale=0.33]{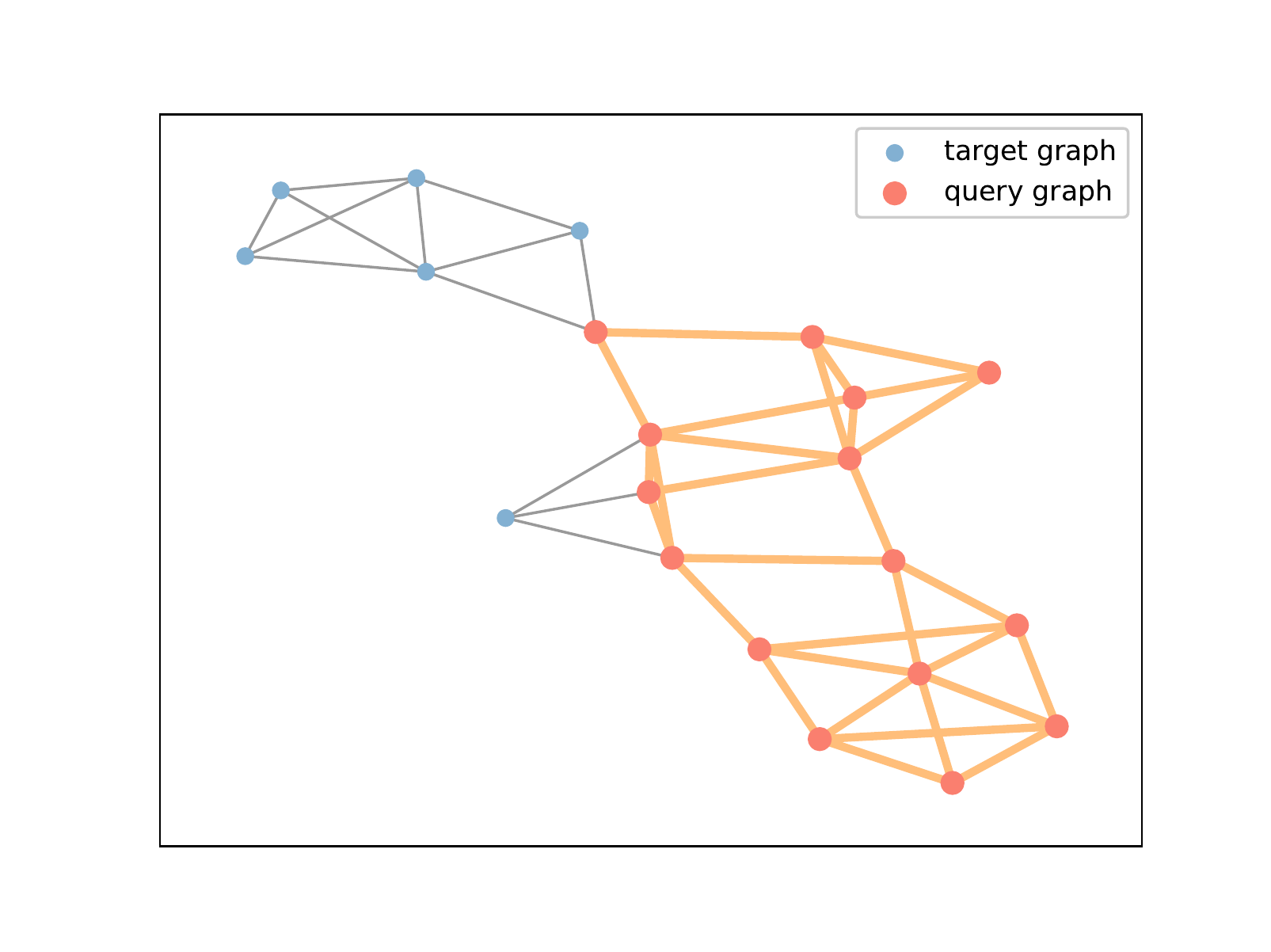}
		\centerline{\scriptsize{(a)}}
	\end{minipage}
        \hspace{.08in}
	\begin{minipage}[b]{0.3\linewidth}
		\centering
		\includegraphics[scale=0.33]{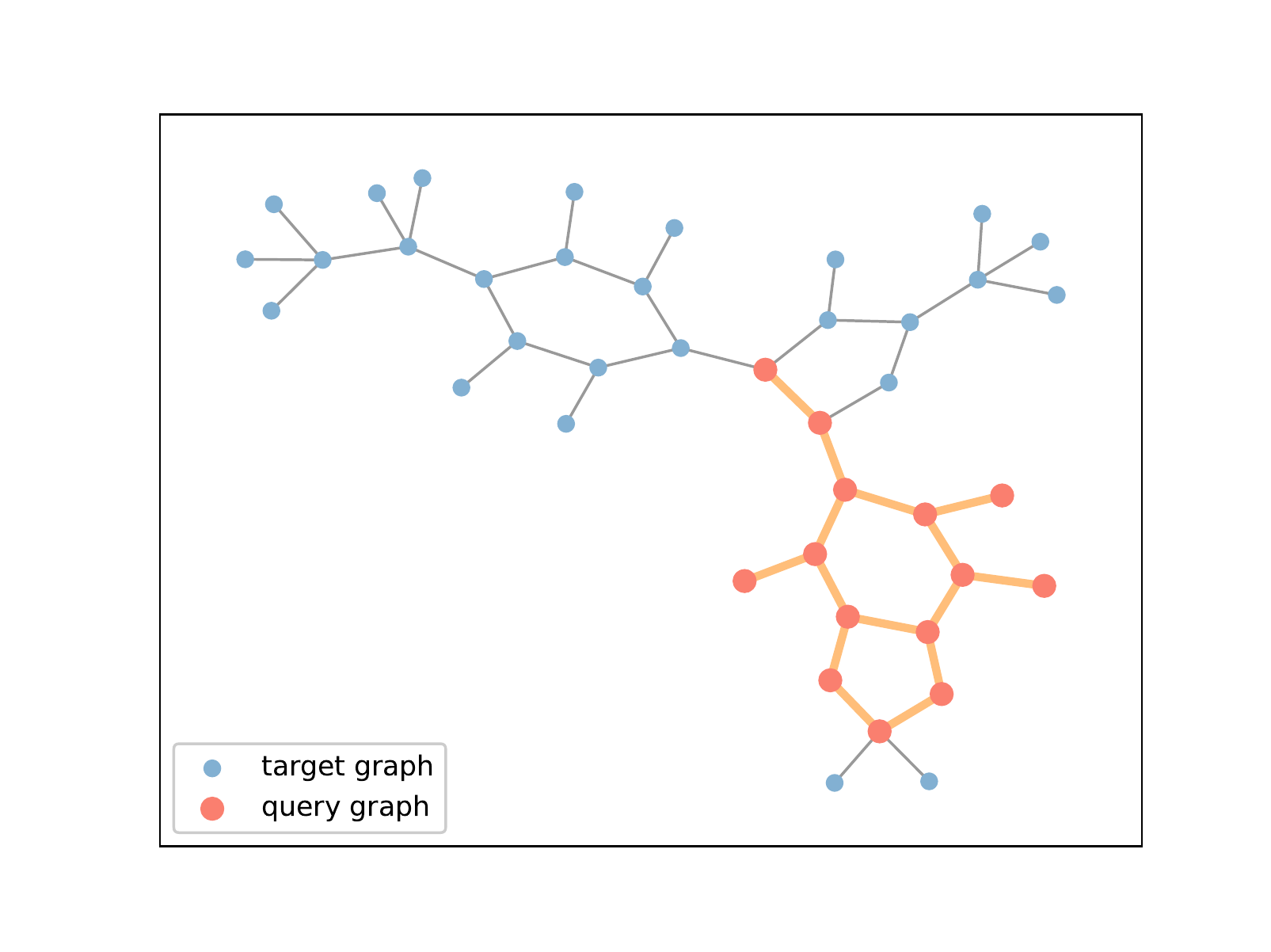}
		\centerline{\scriptsize{(b)}}
	\end{minipage}
    \hspace{.08in}
	\begin{minipage}[b]{0.3\linewidth}
		\centering
		\includegraphics[scale=0.33]{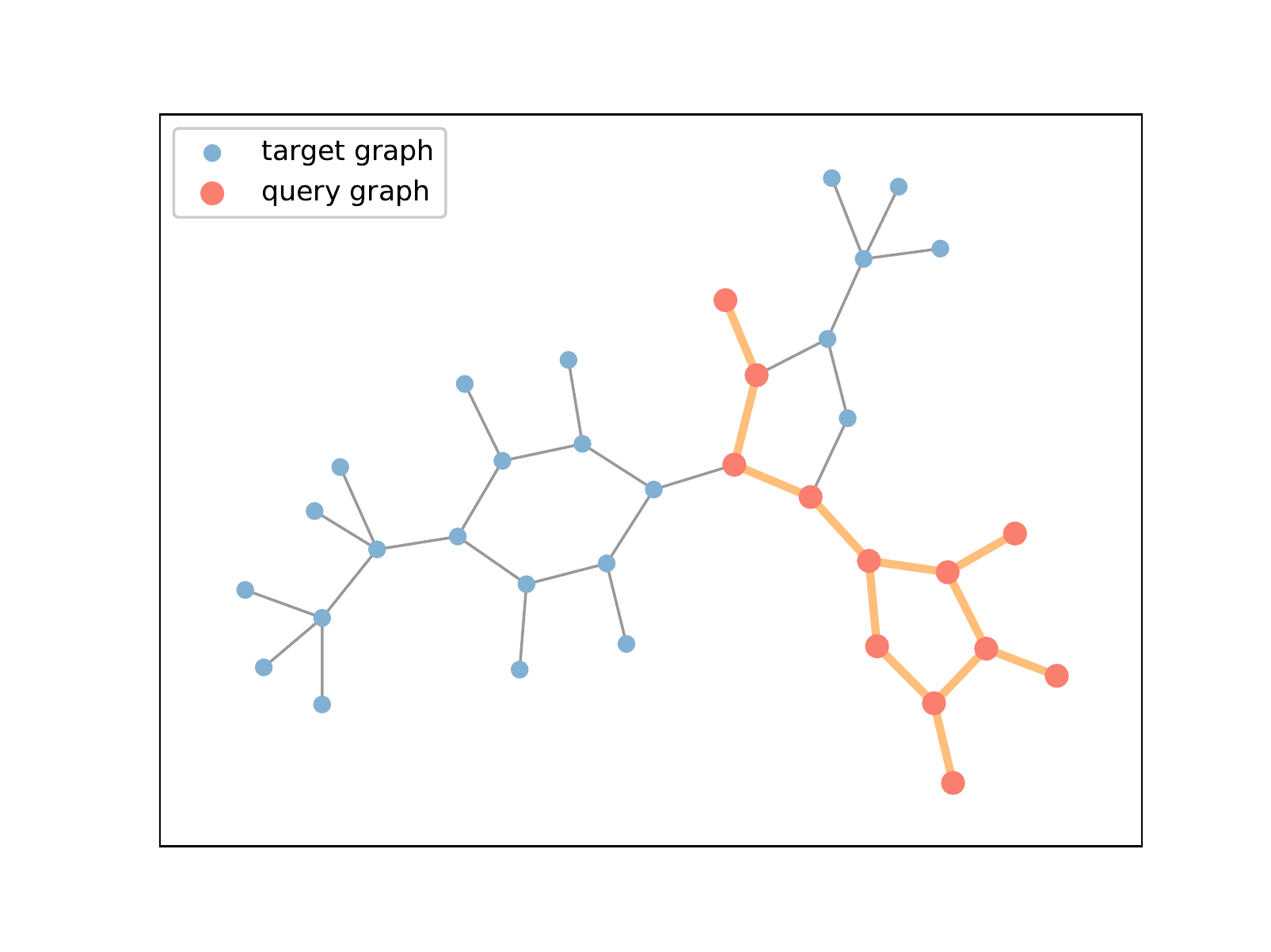}
		\centerline{\scriptsize{(c)}}
	\end{minipage}
	\caption{\footnotesize The detected subgraphs by $D^2$Match}
	\label{fig:visual}
\end{figure*}

\subsection{\ourmeth at Work}
\label{visual_subgraph}
Recall that \ourmeth learns an indicator matrix to capture pairwise similarities.
It plays the role of permutation matrix in matching, allowing us to pinpoint the matched subgraph. 
This is particularly useful since the exact position is required for some downstream applications such as web search.
In comparison, other learning-based methods are unable to pinpoint local correspondences, but only establish the existence of a matching.
We provide a visualization of the matched subgraph to better understand the problem difficulty and the effectiveness of our method, as shown in the Fig.~\ref{fig:visual}.

\subsection{Ablation Studies}
\label{subsec:ablation}

We perform ablation studies for the GNN module, subtree module, and chordless cycle.

The GNN module in our analysis will capture the distributional features on the graph, such as the edge density difference between classes.
The GNN module is thus essential for datasets with multiple distributions, also called biased data.
We run experiments on both the biased and unbiased synthetic datasets to show the performance of our method and its variation that without the GNN module, as shown in Table \ref{tab: subtree ablation}. 
\ourmeth outperforms \ourmeth without the GNN module as our theory predicts.
But our GNN module shares the same weaknesses as the other GNN models when dealing with evenly distributed data.
We observe that the performance of the GNN module drops significantly on hard datasets similar to other GNN models.
The subtree module can significantly improve the performance because it harnesses the property of subgraph-matched data, making it robust to data's distribution. 
Our subtree module outperformed the GNN module on all datasets in our ablation study, demonstrating its effectiveness.

We also perform the ablation study on the Synthetic dataset to test the effect of chordless cycles,as shown  in Table~\ref{tab: cc ablation}. Results show the chordless cycles boost the performance with limited extra time consumption.

\begin{table}
\footnotesize
\caption{ The ablation study of D2Match module }
    \centering
    \scriptsize
    \scalebox{1.0}{
    \begin{tabular}{c|cccccccccc}
    \toprule
            & Synthetic &  Synthetic$^+$ & Proteins & Proteins$^{*}$ & IMDB-Binary & IMDB-Binary$^{*}$ & FirstMMDB & FirstMMDB$^{*}$ \\
        \midrule
              \ourmeth(gnn only) & 61.1$_{\pm\text{13.31}}$ & 70.2$_{\pm\text{18.58}}$  & 95.2$_{\pm\text{1.04}}$ & 77.2$_{\pm\text{8.11}}$ &  50.0$_{\pm\text{0.00}}$ & 64.4$_{\pm\text{19.73}}$ & 69.7$_{\pm\text{26.98}}$ & 67.8$_{\pm\text{24.38}}$   \\
              \ourmeth(subtree only) & 70.0$_{\pm\text{2.09}}$ & 74.8$_{\pm\text{2.56}}$  & 100.0$_{\pm\text{0.00}}$ & 82.0$_{\pm\text{2.92}}$ &  92.9$_{\pm\text{1.04}}$ & 82.8$_{\pm\text{4.02}}$ & 100.0$_{\pm\text{0.0}}$ & 72.0$_{\pm\text{6.20}}$   \\       
       \midrule
            \ourmeth & 72.7$_{\pm\text{4.45}}$ & 86.6$_{\pm\text{1.44}}$  & 100.0$_{\pm\text{0.00}}$ & 83.4$_{\pm\text{2.97}}$ &  93.3$_{\pm\text{1.03}}$ & 90.2$_{\pm\text{1.79}}$ & 100.0$_{\pm\text{0.0}}$ & 86.4$_{\pm\text{7.44}}$   \\    
        \bottomrule
    \end{tabular}
    }
    \label{tab: subtree ablation}
\end{table}

\begin{table}[!th]
\hspace{-0.7cm}
\begin{minipage}[!hbtp]{.50\linewidth}
    \centering
    \vspace{0.3cm}
    \caption{ The ablation study of cc }
    \scalebox{0.9}{
    \begin{tabular}{ccc}
    \toprule
            & Synthetic &  RunTime \\
        \midrule
              \ourmeth  & 74.3$_{\pm\text{1.60}}$ & 19.7s/epoch    \\
            \ourmeth(w/o cc) &  72.7$_{\pm\text{4.45}}$ & 10.3s/epoch \\
        \bottomrule
    \end{tabular}
    }\vsa
    \label{tab: cc ablation}  
    \caption{ Random seed comparison }
\scalebox{1.1}{
    \begin{tabular}{c|cc}
    \toprule
            & proteins &  mutag \\
        \midrule
              Seed(0) & 100.0$_{\pm\text{0.00}}$ & 100.0$_{\pm\text{0.00}}$   \\
              Seed(1) & 100.0$_{\pm\text{0.00}}$ & 100.0$_{\pm\text{0.00}}$   \\
              Seed(2) & 100.0$_{\pm\text{0.00}}$ & 100.0$_{\pm\text{0.00}}$   \\
       \midrule
            Fixed & 100.0$_{\pm\text{0.00}}$ & 100.0$_{\pm\text{0.00}}$   \\
        \bottomrule
    \end{tabular}
}\vsa
    \label{tab: random seed result}
\end{minipage}
\hspace{-1.0cm}
\begin{minipage}[!hbtp]{.5\linewidth}
    \centering
    \caption{Runtime analysis }
    \scalebox{0.86}{
    \begin{tabular}{c|cc}
    \toprule
            & Training(s/epoch) & Inference(s/epoch)  \\
        \midrule
              SimGNN & 1.732 & 0.385 \\
              NeuroMatch & 2.234 & 0.311 \\
              GMN-embed & 1.850 & 0.290 \\
              GraphSim & 3.223 & 0.433 \\
              IsoNet & 10.553 & 1.939 \\
       \midrule
            \ourmeth-Subtree(S=2) & 2.940 & 0.456  \\
            \ourmeth-Subtree(S=3) & 3.889 & 0.581  \\
            \ourmeth-Subtree(S=4) & 4.410 & 0.673  \\
            \ourmeth-Subtree(S=5) & 5.143 & 0.750  \\
            \ourmeth-GNN & 2.678 & 0.495  \\
            \ourmeth & 8.163 & 1.114  \\
        \bottomrule
    \end{tabular}
    }\vsa
    \label{tab: time analysis}  \vsc\vsc   
\end{minipage}
\end{table}

\subsection{Random Effect}
Although our experiments do not rely on random seeds, a random split may affect the results. To test this, we set up several random seeds and permute the raw data order before getting the five-fold. We experiment on the Protein and Mutag datasets with trivial random seed 0,1,2 and obtain nearly identical performance. See Table \ref{tab: random seed result}.

While other methods based on GNNs tend to capture the divergence of distributions in the training set and hence are easily affected by randomness, our subtree module performs the matching explicitly by the degeneracy property, as opposed to modeling the data distribution in others, hence ours is insensitive to data partitioning.

\begin{table}[th!]
    \centering
    \caption{ The hard dataset details }
    \scalebox{0.94}{
    \begin{tabular}{c|cccccccccc}
    \toprule
            & Synthetic$^+$ &  Proteins$^{*}$ & Mutag$^{*}$ & Enzymes$^{*}$ & Aids$^{*}$ & IMDB-Binary$^{*}$ & Cox2$^{*}$ & FirstMMDB$^{*}$ \\
        \midrule
              Average nodes (target) & 40.0 & 38.8  & 18.2 & 31.5 &  14.7 & 19.0 & 41.3 & 1376.7   \\
       Average nodes (query) &  15.0 & 11.4 & 9.1 & 15.4 & 4.4 & 14.2 & 15.0 & 15.0 \\
       \midrule
        Average edges (target) &  259.5 & 146.7 & 40.2 & 120.6 & 30.0 & 177.1 & 87.0 & 6141.6  \\
        Average edges (query) &  67.3  & 35.5 & 17.6 & 52.6 & 7.1 & 102.6 & 29.9 & 45.6  \\
        \bottomrule
    \end{tabular}
    }\vsa
    
    \label{tab: hard dataset}
\end{table}

\begin{table}[th!]
\footnotesize
    \centering
    \caption{ The dataset details }
\scalebox{1.1}{
    \begin{tabular}{c|cccccccccc}
    \toprule
            & Synthetic &  Proteins & Mutag & Enzymes & Aids & IMDB-Binary & Cox2 & FirstMMDB \\
        \midrule
              Average nodes (target) & 40.0 & 39.1  & 17.9 & 33.0 &  15.7 & 19.8 & 41.3 & 1376.5   \\
       Average nodes (query) &  15.0 & 14.4 & 9.0 & 14.8 & 7.9 & 14.6 & 14.4 & 15.0 \\
       \midrule
        Average edges (target) &  241.7 & 146.5 & 39.5 & 125.6 & 32.4 & 193.1 & 87.0 & 6144.3  \\
        Average edges (query) &  50.6  & 68.9 & 25.1 & 75.3 & 17.1 & 141.0 & 42.8 & 68.1  \\
        \bottomrule
    \end{tabular}
}\vsa
    
    \label{tab: dataset}
\end{table}

\subsection{Runtime Analysis}
\label{sec:runtime}
We add the runtime analysis experiment as follows. We compare our method with baselines on the synthetic dataset and record the training and inference time (second) per epoch. The results are shown in Table \ref{tab: time analysis}.

 Our model is slower than some strong baselines like SimGNN and NeuroMatch in the experiment because they deal with the graph-level representations. Our model is faster than IsoNet, which performs edge-level matching.

We conduct an additional ablation study to explore the time consumption of each module in our model. The results show that the time consumption of our model mainly comes from the sampling in the subtree module whose running time is linearly related to the sampling number. When we set the sampling number as 2, the running time is on par with the others. Furthermore, the running time for the GNN module is the same as for the other baselines. 
In sum, we observe that our model's scalability is acceptable as both complexity analysis and empirical running time show ours is slower than others only by a constant factor.

\subsection{Dataset Details}
\label{subsec:dataset}

We describe the average node number and average edge number of the target graph and query graph in the Table~\ref{tab: dataset} and Table~\ref{tab: hard dataset}. Except the hard datasets, we generate 1000 graph pairs for Synthetic, Proteins, Mutag, Enzymes, Cox2 and FirstMMDB and 2000 graph pairs for Aids and IMDB-Binary which have smaller graph size. 
For the hard dataset, we uniformly generate 500 graph pairs.

\subsection{Results on More Datasets}
\label{subsec: more experiments}
We conduct experiments on the OGB benchmark dataset~\citep{hu2020ogb}, including Ogbg-molhiv and Ogbg-molpcb. We follow the same strategy in the paper to construct normal and hard versions for these datasets and choose the best-performing baselines for comparison, including SimGNN and NeuroMatch. We present new results in Table~\ref{tab: obg result}.

We find that our model performs slightly better than others on normal datasets while gaining a significant advantage over baselines on hard datasets. These results are consistent with our previous experiments, demonstrating that our model exploits the subgraph matching property, rather than simply modeling the divergence of the data distribution as other GNNs.
\eat{
\begin{table}
\footnotesize
\caption{ Obg dataset performance comparison in terms of accuracy }
    \scriptsize
    \centering
    \begin{adjustbox}{width=\textwidth}
    
    \end{adjustbox}
    \label{tab: obg result}
\end{table}
}

We experiment on continuous features from the MNIST, CIFAR10 and PPI datasets, as these are constructed from vision data\citep{Vijay2020Benchmark} or biological information data\citep{Marinka2017PPI}. We  As expected, our model achieves consistent performance as well. See Table \ref{tab: continues result}.

\begin{table}[!th]
\begin{minipage}{.60\linewidth}
    \centering
    \caption{ Obg dataset performance comparison }
\scalebox{0.8}{
    \begin{tabular}{c|cccc}
    \toprule
            & ogb-molhiv &  ogb-molhiv$^{*}$ & ogb-molpcba & ogb-molpcba$^{*}$ \\
        \midrule
              SimGNN & 99.4$_{\pm\text{0.65}}$ & 81.6$_{\pm\text{2.70}}$  & 99.8$_{\pm\text{0.27}}$ & 86.2$_{\pm\text{2.28}}$  \\
              NeuroMatch & 98.3$_{\pm\text{1.68}}$ & 86.0$_{\pm\text{3.54}}$  & 99.8$_{\pm\text{0.27}}$ & 90.6$_{\pm\text{3.51}}$   \\       
       \midrule
            \ourmeth & 99.8$_{\pm\text{0.27}}$ & 99.6$_{\pm\text{0.54}}$  & 100.0$_{\pm\text{0.00}}$ & 100.0$_{\pm\text{0.00}}$  \\    
        \bottomrule
    \end{tabular}
}\vsa
    \label{tab: obg result}  
\end{minipage}
\begin{minipage}{.4\linewidth}
    \centering
    \caption{ Continues dataset performance }
\scalebox{0.8}{
    \begin{tabular}{c|ccc}
    \toprule
            & Cifar10 & MNIST & PPI  \\
        \midrule
              SimGNN & 89.0$_{\pm\text{21.82}}$ & 98.5$_{\pm\text{0.93}}$ &  77.0$_{\pm\text{24.67}}$ \\
              NeuroMatch & 98.1$_{\pm\text{1.14}}$ & 95.9$_{\pm\text{1.34}}$ & 50.0$_{\pm\text{0.00}}$   \\
       \midrule
            \ourmeth & 99.3$_{\pm\text{0.27}}$ & 98.8$_{\pm\text{1.15}}$ & 98.8$_{\pm\text{1.06}}$   \\
        \bottomrule
    \end{tabular}
}\vsa
    \label{tab: continues result}  \vsc\vsc   
\end{minipage}
\end{table}

\subsection{Comparison with Exact Methods}
\label{sec: exact compare}
we compare exact matching solutions, including VF2 and ISMAGS.
By nature, we know that all the exact matching methods can obtain 100 \% accuracy.

As a trade-off between accuracy and execution time, we make the comparison inspired by the setup in NeuroMatch~\citep{Rex2020NeuroMatch}. We say an execution succeeds when its run time is less than 60s. We compare the success rate of the exact methods by varying the query graph size from 10 to 50 on the synthetic data and report the results in Figure~\ref{fig:exact compare}.

\begin{figure}
\centering
\includegraphics[ width=0.8\textwidth]{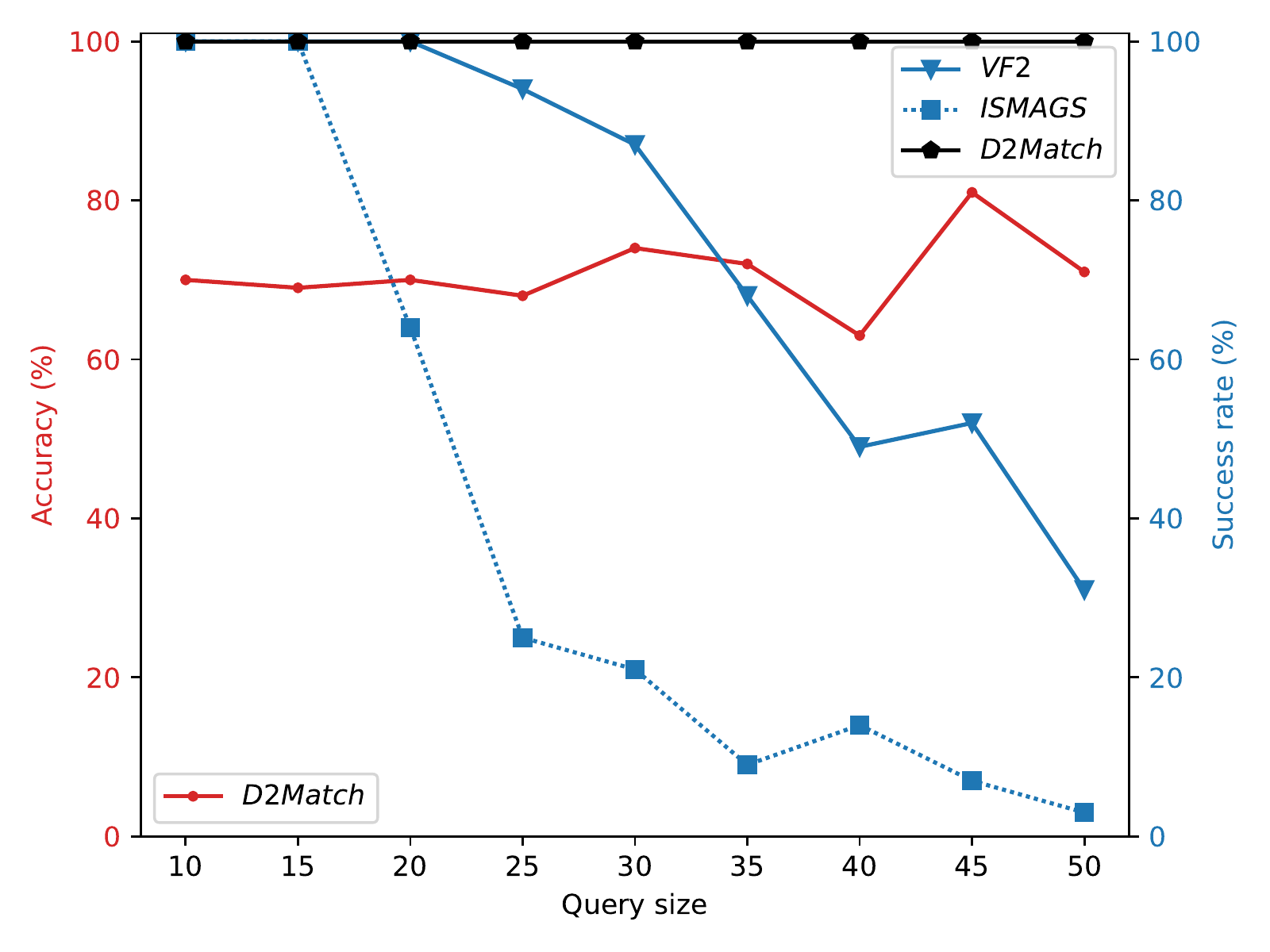}\vsa 
\caption{\label{fig:exact compare} \footnotesize Comparison with exact methods} \vsa 
\end{figure}

We show in our experiment that the failure of exact matching methods increases significantly when the target graph has more than 30 nodes, compared to our stable performance, indicating the incompetence of these methods on large-scale datasets.

\end{document}